\documentclass[twoside,11pt]{article}

\usepackage{jmlr2e}

\usepackage{amsmath}
\usepackage{graphicx}
\usepackage{algpseudocode}
\usepackage{bbm}
\usepackage{enumitem}
\usepackage{xcolor}
\usepackage{mathtools}
\usepackage{float}
\usepackage{amssymb}
\usepackage{algorithm}
\usepackage{thmtools}
\usepackage{thm-restate}

\DeclareMathOperator*{\argmax}{argmax}
\DeclareMathOperator*{\sat}{sat}
\usepackage{mathtools}
\DeclarePairedDelimiter{\ceil}{\lceil}{\rceil}

\jmlrheading{99}{2020}{1-42}{9/19}{02/20}{meila00a}{Or Raveh and Ron Meir}

\ShortHeadings{PAC Guarantees for Multi-Agent System with Restricted Communication}{Or Raveh and Ron Meir}
\firstpageno{1}

\begin{document}
\begin{sloppy}

\title{PAC Guarantees for Cooperative Multi-Agent Reinforcement Learning with Restricted Communication}

\author{\name Or Raveh \email sravehor@campus.technion.ac.il \\
       \addr Viterbi Faculty of Electrical Engineering\\
       Technion, Israel
       \AND
       \name Ron Meir \email rmeir@ee.technion.ac.il \\
       \addr Viterbi Faculty of Electrical Engineering\\
       Technion, Israel
       }
\editor{Aryeh Kontorovich and Gergely Neu}

\maketitle

\begin{abstract}
%
Cooperative multi-agent systems have obvious advantages over single-agent systems due to shared information. However, an important limitation is the presence of restrictions over the communication channels between different agents, ubiquitous in real-world systems. The effects of such restrictions on the performance of exploration in multi-agent Reinforcement Learning have not been investigated thoroughly theoretically. We develop a theoretically motivated algorithm based on model-free Reinforcement Learning for cooperative multi-agent systems, in a framework that allows noisy and resource limited sparse communication between agents. By analyzing this algorithm, we provide the first (to the best of our knowledge) PAC bounds for multi-agent Reinforcement Learning in this setting. Our algorithm optimally combines information from the resource-limited agents, thereby analyzing the tradeoff between noise and communication constraints and quantifying their effects on the overall system performance.
\end{abstract}

\begin{keywords}
  Reinforcement Learning, Cooperative Multi Agent Systems, Restricted Communication
\end{keywords}

\section{Introduction}\label{sec:Introduction}

In recent years, there has been an increased interest in distributed artificial intelligence, and particularly in multi-agent systems (MAS) learning, such as Learning in Games \citep{MAReview}. MAS further divides into competitive MAS, where agents have contradicting goals \citep{silver2018general,bansal2017emergent}, and cooperative MAS in which agents cooperate in order to perform some task, such as information gathering \citep{viseras2018distributed}, traffic light control \citep{srinivasan2006cooperative} and navigation \citep{robinson2002using}. Cooperative MAS have obvious advantages over single-agent approaches, as the multiple agents can share information and assume designated roles that achieve the goal in a more effective manner \citep{panait2005cooperative}. It can be either centralized, in which case there is a single entity performing learning \citep{Pazis2016}, or decentralized where there is more than one learner \citep{omidshafiei2017deep}. While there are some environments in which a decentralized approach is preferable, there are also cases where the centralized approach is advantageous, due to the fact that the single learner has information regarding the whole system \citep{panait2005cooperative}. 

While cooperative MAS have obvious advantages over single-agent systems due to shared information, an important limitation is the presence of restrictions over the communication between different agents which, in some cases, can even harm  the system's performance \citep{NoisyControl}. Such constraints can result from various reasons, such as thermal noise \citep{baran2010modeling}, quantization error \citep{walden1999analog}, latency \citep{roth2007execution}, communication failure \citep{akyildiz2002wireless} etc. While some theoretical results for cooperative MAS disregard such restrictions, these appear  to be not only an important, but also a necessary consideration, as it has been argued that multi-agent systems with perfect communication are analogous to a single agent system with multiple effectors \citep{stone2000multiagent}. In order to resolve this issue, many works assume a sparse or local communication setting, such as gossip algorithms \citep{korda2016distributed} and neighbors-only communication \citep{martinez2018decentralized}, in the hope of helping to reduce the burden over the communication network \citep{gupta2000capacity}. However, the effects of sparse communication on the performance of RL algorithms and the relation to the communication noise strength has not been investigated thoroughly theoretically, and a fundamental question is in which cases sparse communication is indeed advantageous to solving the task. 

While there are various learning methods that can be used in cooperative MAS systems, we focus on Reinforcement Learning (RL) which has demonstrated impressive recent success in many domains, notably game playing \citep{silver2018general}, and, more specifically, on Multi Agent RL (MARL). In standard Model Free (MF) RL, a learner aims to develop an optimal policy based on interaction with the environment through state observations and action selection. While there have been many heuristic algorithms suggested for cooperative MARL \citep{Dimakopoulou2018a, Dimakopoulou2018, foerster2016learning}, theoretical results are lacking, and there are even fewer of them in the restricted communication setting. Works analyzing the exploration efficiency of RL algorithms usually use either Probably Approximately Correct (PAC) performance measures \citep{Pazis2016, Dann2017, Lattimore2014}, or regret bounds \citep{Osband2014}. In this paper we focus on the former, in which there are a few results known for the unlimited communication MARL case \citep{Guo2015,Pazis2016a}, but no results for the restricted communication case have been established to the best of our knowledge. 

More concretely, we focus on cooperative MARL with a centralized learner, see Figure \ref{fig:model_diagram}. Consider an agent learning to act in a stochastic environment. Unlike the standard RL setting, the learner has direct access to the environment only through a set of $N$ agents, acting concurrently. Based on the state/action/reward information received from the agents, the learner constructs a $Q$-function representing its estimate for the value of the states and actions, and tries to learn an optimal policy that maximizes its expected cumulative reward. In addition to the centralized learner-agents interaction, the agents themselves interact through a (fully connected or sparse) connectivity graph. Motivated by the ubiquity of communication constraints in real-world problems discussed earlier, we allow learner-agent and agent-agent communications to be limited by channel noise and bandwidth restrictions. While in the special case of perfect communication there is no need for inter-agent communication given a single learner, we show in this work that, under restricted communication, inter-agent communication enables improvement of performance even in the case of is a single learner.

The main contributions of the present work are the following.  \emph{(i)} We provide the first (to the best of our knowledge) PAC bounds for MARL in the noisy communication setting. \emph{(ii)}  We suggest an algorithm that optimally combines information from the agents in order to reduce the effects of communication noise, and demonstrate its performance on a simple problem. \emph{(iii)} We study the effects of limited communication resources on performance, and highlight the tradeoff between specific kinds of communication noise and the sparsity of the communication network, and provide provably effective means to balance them.

The special case of perfect communication cooperative MARL was theoretically analyzed in PAC terms in \citep{Guo2015,Pazis2016}. More recently, empirical work by  \citep{Dimakopoulou2018a, Dimakopoulou2018} suggested necessary properties for effective cooperative MAS exploration, specifically commitment, diversity and adaptivity, and demonstrated that exploiting them improves exploration. Other interesting work focuses on adaptively learning how to perform communication between agents, in addition to learning the desired task \citep{Sukhbaatar2016,Mordatch2017,Jiang2018}. It was also recently demonstrated that single-agent meta-learning can be improved by using parallel agents which coordinate their exploration \citep{Parisotto2019}. Another recent theoretical work \citep{Doan2019} demonstrated finite time convergence of the TD(0) algorithm for the distributed multi-agent case. 

\section{Background, model, and definitions}
\label{sec:Background, model, and definitions}
Consider a set of $N$ agents acting concurrently in the same environment and transmitting their state-action information to a learner. For each agent, the setup is modeled as a $5$-tuple $(S,A,P,R,\gamma)$. Here $S$ is the finite state space, $A$ is the finite action space, $P$ is the dynamics probability law (given state $s$ and action $a$, the probability to traverse to state $s'$ is $p(s'|s,a)$), $R(s,a,s')$ is the reward function, and $\gamma$ is the discount factor\footnote{The results in this paper hold, with appropriate modifications, for a a continuous state space, but we focus for simplicity on the finite state space.}. This is a specific case of a dynamic Stochastic Game \citep{MAReview} in which agents are cooperative (having the same reward function) and have the same set of actions. The $Q$-function under policy $\pi$ for some agent is defined as $Q^{\pi}(s,a)=E^{\pi}\left[\sum_{t=0}^{\infty} \gamma^{t}R(s_{t},a_{t},s_{t+1})|s_{0}=s,a_{0}=a\right]$, and the optimal $Q$-function is $Q^{*}(s,a)=\max_{\pi}Q^{\pi}(s,a)$. The Bellman operator for $\pi$ is $B^{\pi}Q(s,a)=\sum_{s'} p(s'|s,a)[R(s,a,s')+\gamma Q(s',\pi(s'))]$. Similarly, the value function for $\pi$ is defined as $V^{\pi}(s)=E^{\pi}[\sum_{t=0}^{\infty} \gamma^{t}R(s_{t},a_{t},s_{t+1})\allowbreak|s_{0}=s]$. 
We also assume that all $Q$-functions are bounded in the interval $[0,Q_{\max}]$, where $Q_{\max}=R_{\max}/(1-\gamma)$. 
\begin{figure}
\includegraphics[width=0.5\textwidth]{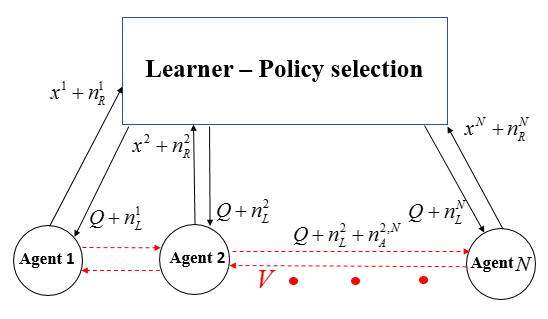}
  \centering
  \caption{\label{fig:model_diagram}A diagram representing the model. Agents collect world observations $x^{i}_{t}\triangleq (s^{i}_{t},a^{i}_{t},R^{i}_{t},s^{',i}_{t})$, and send them to a common learner. The learner performs value iteration and returns to the agents $Q$, a representing $Q$-function table. The agents perform a greedy policy over the received $Q$-function. In Our model, the $Q$-function $Q$ received from the learner by agent $i$ is noisy with some learner-agent noise $n_{L}^{i}$. $Q$-function table based communication between agents is allowed for directly connected agents on graph $\Gamma$, with an addition of agent-agent noise termed $n_{A}^{j,i}$ for agent $j$ to agent $i$. The sample $x_{i}$ sent to the learner by agent $i$ is also noisy, with a reward noise termed $n^{i}_{R}$.}
\end{figure}

As shown in Figure \ref{fig:model_diagram}, Each of the $N$ agents interacts with a MDP composed of \emph{unknown} transition probabilities $P$ and rewards $R$, and acts according to some policy within that environment. The agents send their observations to a joint learner who performs value iteration, and produces a $Q$-function table $Q$ with value $Q(s,a)$ for each state-action. The learner sends each of the agents a copy of its $Q$-function table, over which they perform a greedy policy.\footnote{It is enough to only send agent $i$ its $Q$-function value, and values for the other agents directly connected to it on $\Gamma$.} As stressed in Section \ref{sec:Introduction} a more realistic scenario is that in which communication between components is restricted or noisy. We consider a setting where communication restrictions modeled as additive noise are present (for more discussion on the types of noise this model holds for see Section \ref{sec:Communication noise exploration bounds}). 
Agent-learner noise termed $n_{R}^{i}$ is added to the rewards being sent from agent $i$ to the learner, where for a given reward $R^{i}(s,a,s')$ collected by agent $i$,  
\begin{equation}
\label{eq: reward noise}
   \widetilde{R}^{i}(s,a,s') = R^{i}(s,a,s')+n_{R}^{i}(s,a) \quad\forall (s,a) 
\end{equation} 
The learner has access to previously agent-collected samples which are used for value iteration. For each state-action $(s,a)$, the \emph{sample set} $u(s,a)$ denotes a collection of previously collected noisy rewards and next-state values, sent by agents that passed through $(s,a)$. In the opposite direction, learner-agent noise termed $n^{i}_{L}$ is present and adds to the $Q$-function being sent from the learner to agent $i$, resulting in a $Q$-function termed $\widetilde{Q}_{L}^{i}$. We also allow direct communication between pairs of agents, assuming that the agents lie on some given graph $\Gamma$ represented by a $0-1$ adjacency matrix $E$, so that every pair of directly connected agents $i,j$ such that $e_{ij}=1$ can send each other their own copy of the $Q$-function received from the learner, with an added agent-agent noise termed $n^{j,i}_{A}$ from agent $j$ to $i$. Such $Q$-functions are termed $\{\widetilde{Q}_{A}^{j,i}\}_{j \neq i}$. We denote the degree of agent $i$ on this (undirected) graph as $d_{i}$. Since the agent-agent communication channel adds further noise to the $Q$-function, there is a tradeoff between using the more precise $Q$-function sample received from the learner, and  the less precise $Q$ values received from the other agents. In Section \ref{sec:Estimation methods} an estimation method is introduced that allows each agent to obtain an improved \emph{estimated $Q$-function} with an effective noise term $n^i(s,a)$, denoted $\hat{Q}^{i}$, mitigating the inter-agent communication noise. Next, we formally define the different $Q$ functions used. 
%
\begin{equation}
    \begin{aligned}
    \label{eq: learner and agent noise}
      \widetilde{Q}^{i}_{L}(s,a) &= Q(s,a)+n^{i}_{L}(s,a) \\  \widetilde{Q}^{j,i}_{A}(s,a) &= Q(s,a)+n^{j}_{L}(s,a)+n^{j,i}_{A}(s,a) \quad \forall i,j \neq i \\
      \hat Q^i(s,a) &= Q(s,a)+n^{i}(s,a) 
    \end{aligned}
\end{equation}
In order to quantify the sample complexity of the algorithm, we use the \emph{Total Cost of Exploration} (TCE), Definition \ref{def:TCE}. This is the difference between the value function of the optimal policy and that of the policy used by the algorithm, summed over all times steps, and can easily be shown to dominate the number of sub-optimal steps. A single-agent algorithm is  \emph{efficient PAC-MDP} if its sample complexity, space complexity and computational complexity are all bounded by some polynomial in the quantities of the problem (see Definition \ref{def:PAC}), and a multi-agent algorithm is efficient-PAC MDP if the above holds uniformly for all agents. Other definitions we use throughout this work are introduced in Appendix \ref{sec:Appendix A: Efficient PAC proofs}. 
%
\section{Communication noise tolerant PAC Exploration}
\label{sec:Communication noise tolerant PAC Exploration}

We begin with an intuitive overview of the algorithm, described in pseudo-code in Algorithms \ref{alg:Communication noise tolerant PAC Exploration} and \ref{alg: Advance N agents}. The learner maintains a list of currently collected tuples of (state, action, reward, next state), from all agents,  as well as its $Q$ function estimate based on value iteration. The agents receive a noisy version of the learner's $Q$ function, as well as noisy $Q$ function values from other agents \eqref{eq: learner and agent noise}, construct an estimated $Q$ function, update their tuples and transmit them to the learner. 
\begin{algorithm}[t]
	\caption{Communication-noise tolerant PAC Exploration}\label{alg:Communication noise tolerant PAC Exploration}
	\begin{algorithmic}[1]
		\State \textbf{Initialization} Initialize learner's list to empty, parameters $k$,$k_{m}$,$\epsilon_{a}$, set $Q(s,a)=Q_{max}$ for every $(s,a)$, and sample sets $u(s,a),u^{tmp}(s,a)$ to empty, transmit $Q$ function to agents and set update to true. 
		\For {$t=1,2,\ldots$}
		   \State Advance $N$ agents using Algorithm $2$.
		   \For {agent $i=1,2,\ldots,N$}
			    \Comment{The learner updates the sample sets}
			\State Remove $(s_{i,t},a_{i,t},\widetilde{R}_{i,t},s_{i,t+1})\triangleq (s,a,\widetilde{R},s')$ from the list 
			\If { $|u(s,a)|<k$}
			    \State Add the current sample to $u^{tmp}(s,a)$ 
			    \If {$|u^{tmp}(s,a)|>|u(s,a)|$ \textbf{and} $|u^{tmp}(s,a)| = 2^{p}k_{m}$  where $p \in \mathbb{N}$}
		            \State Set $u(s,a) \gets u^{tmp}(s,a)$ and $u^{tmp}(s,a) \gets \varnothing$
		            \State set update to true 
			    \EndIf
			\EndIf
			\EndFor
			\If {update is true}
			    \While {$\exists (s,a)$  s.t $|\widetilde{B}Q(s,a)-Q(s,a)|> \epsilon_{a}$} \Comment{Value iteration}
				    \State Set $Q \gets \widetilde{B}Q$   
				\EndWhile
				\State Transmit $Q$ function to agents.  
			\EndIf					
	    \EndFor
	\end{algorithmic}
\end{algorithm} 
The contribution of this algorithm is the introduction of an effective cooperative multi-agent exploration scheme in the presence of restricted communication between different system components (manifested through the estimation of an improved $Q$ function). 

In Algorithm \ref{alg: Advance N agents}, $N$ agents receive noisy $Q$-function samples from the learner $\{\widetilde{Q}_{L}^{i}\}$ and adjacent agents  $\{\widetilde{Q}_{A}^{j,i}\}_{j \neq i}$, and linearly estimate the $Q$-function using some given weight matrix $W$ as in line $5$. Each agent then performs a greedy step on the estimated $Q$-function, and sends its collected sample of the environment with a noisy reward to the learner. In Section \ref{sec:Estimation methods} we state the estimation weights resulting in an optimal sample complexity bound.  

In Algorithm \ref{alg:Communication noise tolerant PAC Exploration},  in lines $4-13$ the algorithm updates the sample sets with the newly collected samples awaiting in list $Z$. In order to maintain the collected samples independent, for each state-action, the algorithm first waits for $2^{p}k_{m}$ new samples to be gathered for some parameter $k_{m}$ and integer $p$, before replacing the old sample set $u(s,a)$ with the new samples. To keep the computational and space complexity bounded, we limit the size of the sample set to no more than $k$, where $k/k_{m}$ is some power of $2$. Given that there was an update of some sample set, the algorithm performs value iteration on lines $15-17$ with convergence up to a factor of $\epsilon_{a}$, and then publishes noisy copies of the $Q$-function for all agents. The approximate Bellman operator $\widetilde{B}$ uses the median of means trick to estimate the mean-value for the collected samples \citep{Pazis2016a,Catoni,Sub-Gaussian}, and is defined in Appendix \ref{sec:Appendix A: Efficient PAC proofs}. In this regard, $k_{m}$ is the number of groups we divide the samples to in a sample set, in order to calculate the median. Note that if the learner does not update the $Q$-function on a given time-step, it does not re-publish it, and the agents continue using the latest published table. Recently, it was shown that in order to perform efficient cooperative exploration by multiple agents, a given agent should not change its policy at every time-step, but `commit' to a fixed policy expressing the exploration of a given part of the world for a limited amount of time \citep{Dimakopoulou2018a}, which agrees with our algorithm. 
\section{Communication noise exploration bounds}
\label{sec:Communication noise exploration bounds}
\begin{algorithm}
	\caption{Advance $N$ agents}\label{alg: Advance N agents}
	\begin{algorithmic}[1]
	\State Initialize: Adjacency matrix $E=\left[e_{ij}\right]$ and Weight matrix $W=\left[w_{ij}\right]$
        \For{$i\gets 1$ to $N$, from state s}
                	\If {update is true}{ $\forall (s,a)$} \Comment{Linear estimation of the $Q$-function}
                	    \State Receive $\widetilde{Q}_{L}^{i}(s,a)$ from the learner, and $d_{i}$ samples $\{\widetilde{Q}_{A}^{j,i}(s,a)\}_{j \neq i}$ from agents.
				        \State set $\hat{Q}^{i}(s,a) \gets w_{ii}\widetilde{Q}_{L}^{i}(s,a) + \sum_{j \neq i}^{N}e_{ji}w_{ji}\widetilde{Q}_{A}^{j,i}(s,a)$
				        \State set $\hat{Q}^{i}_{sat}(s,a) \gets \max\{Q_{max},\min\{0,\hat{Q}^{i}(s,a)\}\}$ 
				        \State set update to false.
			        \EndIf
			        \State set update to false. 
				    \State Perform action $a=\argmax_{\widetilde{a}}\hat{Q}_{sat}^{i}(s,\widetilde{a})$.
				    \State Receive reward $R$ and transition to state $s'$.
			\State Transmit a copy of $(s,a,\widetilde{R},s')$ to the a list $Z$ of the learner. \Comment{Reward is noisy}
            \EndFor
	\end{algorithmic}
\end{algorithm} 
The main contribution of this section is to establish PAC guarantees for the noisy communication setting, which we believe to be the first result for this case. Space and computational complexity guarantees and detailed proofs can be found in the Appendix. 
Before presenting the main sample complexity theorem, we describe the noise setup. In this work we focus on two sources of noise - additive noise and quantization noise. A white Gaussian additive noise is usually used to describe the channel noise resulting from thermal fluctuations in the receiver's circuits, and is the main limitation in space and satellite communication \citep{baran2010modeling,walden1999analog}. Quantization noise models analog to digital quantization performed at the receiver, and in some cases can be modeled as an orthogonal noise term which is bounded by the interval between quantization levels 
\citep{walden1999analog, marco2005validity}. Although there are many other sources of communication limitations in distributed systems, such as latency \citep{roth2007execution} and communication failures \citep{akyildiz2002wireless}, in this work we focus on additive noise and quantization noise in order to stress resulting theoretical phenomena resulting from communication constraints, and do not necessarily view our algorithm as intended for some specific application. Each agent $i$ receives a noisy version of the $Q$-function as in \eqref{eq: learner and agent noise} every time the learner updates its real $Q$-function, such that the communication noise $n^{i}(s,a)$ is composed of an additive noise term $b$ and a quantization term $m$, %
\begin{equation}
\begin{aligned}
\label{eq: noise decomposition}
        n^{i}(s,a) &= b^{i}(s,a)+m^{i}(s,a)~. 
\end{aligned}
\end{equation}
Furthermore, the learner receives samples from the agents in which the reward is noisy with an additive noise as in \eqref{eq: reward noise}. In this work we make the following two assumptions. Assumption \ref{assump: Noise_assumptions1} regards independence of the noise terms. \begin{restatable}[]{myassump}{x}
\label{assump: Noise_assumptions1}
The noise term $n^{i}$ in \eqref{eq: noise decomposition} is independent of $Q$ for any state-action $(s,a)$, and independent of any samples collected so far by the agents. The reward noise terms $n^{i}_{R}(s,a)$ are mutually independent for different times, agents and next-state $s'$ variables given $(s,a)$. Furthermore, the reward noise terms $n^{i}_{R}(s,a)$ are zero-mean with a variance no larger than $\sigma_{R}^{2}$. This is true for all noise terms corresponding to $Q$-functions, rewards, state-actions and agents that can be encountered during the run of the algorithm.\footnote{Note that these relations holds for a noise term given a state action $(s,a)$, and does not include the probability of encountering $(s,a)$ itself} 
\end{restatable}
Note that since the noise term $n^{i}$ results from a linear estimation of $Q$-function samples sent from the learner and from different agents, assumption \ref{assump: Noise_assumptions1} results from the fact that the noise terms on the agent-learner and agent-agent channels each obey the independence conditions separately. The reward-noise can model both additive noise or quantization noise, but in this work we focus on an additive noise for simplicity. Assumption \ref{assump: Noise_assumptions2} defines the properties of the noise terms \eqref{eq: noise decomposition} more specifically. 
\begin{restatable}[]{myassump}{x}
\label{assump: Noise_assumptions2}
The noise term $b^{i}(s,a)$ in \eqref{eq: noise decomposition} is sub-Gaussian with mean $0$ and parameter no larger than $\sigma^{i}_{c}$, and $m^{i}(s,a)$ is a bounded random variable with mean $0$ such that $|m^{i}(s,a)|\leq \Delta Q^{i}_{c} $ for some positive numbers $\sigma^{i}_{c}$ and $\Delta Q^{i}_{c}$. This is true for all noise terms corresponding to $Q$-functions, rewards, state-actions and agents that can be encountered during the run of the algorithm.
\end{restatable}
We note that the Sub-Gaussian assumption can be dismissed at the cost of weaker bounds, but this does not affect the derivation of Section \ref{sec:Estimation methods}. In Assumption \ref{assump: Noise_assumptions2}, we model the quantization effect as bounded mean zero additive noise \citep{marco2005validity}. 

Theorem \ref{theo:sample complexity2} below claims that our algorithm is multi-agent efficient-PAC MDP by Definition \ref{def:MAPAC} with $\epsilon=\max_{i}\{\epsilon_{i}^{eff}\}$ and $\delta$, where the $TCE$ of agent $i$ is defined with a parameter $\epsilon_{i}^{eff}$ (see Definition \ref{def:TCE}). As seen in \eqref{eq:VminusV}, this means that the policy produced by the algorithm is close to the optimal policy up to an error of $\epsilon$ for all time steps and all agents, and an additional error (the TCE in \eqref{eq:TCE_thoerem}) which is bounded by terms dependent on the environment and the noise.  We now briefly describe the different sources of error in the bound. \emph{(i)} $\epsilon_{s}$ is the error caused by the finite sample (at most $k$) used to estimate the $Q$-function.  \emph{(ii)} $\epsilon_{a}$  is the error caused by the fact that we may not converge to a fixed point of $\widetilde{B}$ during value iteration, but up to a distance of $\epsilon_{a}$ from it. \emph{(iii)} $\sigma_{R}, \sigma^{i}_{c}, \Delta Q^{i}_{c}$ represent the strength of the two communication noises introduced here, and show that exploration becomes more challenging the larger the noise is. 
\emph{(iv)} $\epsilon_{e}^{i}(t)$ is the error caused by the fact that at time $t$ there may exist state-actions with fewer than $k$ visits. 
\begin{restatable}[]{myassump}{y}\label{assump:main}
Let $(s_{1,i},s_{2,i},...,)$ for $i \in \{1,...,N\}$ be the random paths generated by the $N$ concurrent agents on some execution of Algorithm \ref{alg:Communication noise tolerant PAC Exploration}, and let $\widetilde{\pi}_{i}$ be the (non-stationary) policy followed by agent $i$ in this algorithm.  Let $\epsilon_{b}^{2}=4k_{m}(\sigma^{2}+\sigma_{R}^{2})$ (where $\sigma$ is the variance of the Bellman operator introduced in Definition \ref{def: sigma}), $k_{m}=\ceil{5.6\ln8N\ceil{1+\log_{2}\frac{k}{k_{m}}}(SA)^{2}/\delta}$, $k \geq \ceil{\epsilon_{b}^{2}/((1-\gamma)^2\epsilon_{s}^2)}$, and $\epsilon_{a}$ is defined as in Algorithm \ref{alg:Communication noise tolerant PAC Exploration}. Furthermore assume  $(2\ceil{\frac{1}{1-\gamma}\ln\frac{Q_{\max}}{\epsilon_{s}}}/NSA)\ln(2\ceil{1+\log_{2}(k/k_m)}/\delta<1$.
\end{restatable}
\begin{restatable}[]{mytheo}{samplecomplexityvar}
\label{theo:sample complexity2}
Under Assumptions \ref{assump: Noise_assumptions1}, \ref{assump: Noise_assumptions2}, \ref{assump:main}, define $f=f(N,k,k_{m},\delta,SA) \triangleq \sqrt{2\ln\left(24N\ceil{1+\log_{2}\frac{k}{k_{m}}}(SA)^{3}/\delta\right)}$. Then with probability at least $1-\delta$, for all $t$ and $i$, 
\begin{equation} \label{eq:VminusV}
       V^{*}(s_{t,i})-V^{\widetilde{\pi}_{i}}(s_{t,i}) \leq  \frac{2\epsilon_{a}+2(1+3\gamma)(\Delta Q^{i}_{c}+\sigma_{c}^{i}f)}{1-\gamma}+3\epsilon_{s} +\epsilon_{e}^{i}(t) \triangleq \epsilon^{eff}_{i}+\epsilon_{e}^{i}(t)~,
\end{equation}
where
\begin{equation}
\label{eq:TCE_thoerem}
    \begin{gathered}
       TCE\triangleq \sum_{i=1}^{N}\sum_{t=0}^{\infty}\epsilon_{e}^{i}(t) = \widetilde{O}\left(\left(N\left(Q_{\max}+\sqrt{\sigma^{2}+\sigma^{2}_{R}}\right)+\frac{\sigma^{2}+\sigma^{2}_{R}}{\epsilon_{s}(1-\gamma)}\right)SA/(1-\gamma)\right)~,  
    \end{gathered}
\end{equation}
and $\widetilde{O}$ stands for a big-O up to logarithmic terms.
\end{restatable}
\begin{proof}
We only discuss the proof scheme, and refer the reader to the Appendix for details. The non-stationary policy of each of the agents can be broken up into fixed-policy segments, in which we follow an approximated $Q$-function greedily, where the noise can be shown to be concentrated around its mean. Given that the Bellman error for each such segment is acceptably small, we know that the value function of the greedy policy for that segment has a bounded error with respect to the optimal policy value. We  then use the union bound to show that with a high probability, the Bellman errors of all $Q$-functions for all agents during the run of the algorithm are bounded, and then combine it with with a statement that bounds the number of times we can encounter state-action pairs for which we haven't collected enough samples yet. Combining these results, we can show that the algorithm follows a non-stationary policy with an acceptable error compared to the optimal policy.     
\end{proof}
Notice that while the $Q$-function noise terms only affect $\epsilon_{eff}$, and the reward noise term affects the TCE. While we can sacrifice computational and space complexity to lower $\epsilon_{a},\epsilon_{s}$ to effectively zero, the lowest possible $\epsilon^{eff}_i$ is dictated by the $Q$-function noise terms of the environment. In the next section, we will use the sample complexity bound of Theorem \ref{theo:sample complexity2} to derive estimation methods and algorithms for various cases depending on the nature of the communication noise.
\section{Estimation and optimal agent weighting}
\label{sec:Estimation methods}
In this section, we use the efficient-PAC bound introduced in Theorem \ref{theo:sample complexity2} to derive various algorithms for the noisy communication setting described in Algorithm \ref{alg:Communication noise tolerant PAC Exploration}. In this section we focus on the $Q$-function noise, and so we pose no further assumptions on the reward-noise. Assumption \ref{assumption:quantization noise} treats the various noise terms introduced in the communication channels and essentially guarantees that assumptions \ref{assump: Noise_assumptions1}, \ref{assump: Noise_assumptions2} hold for the effective noise term \eqref{eq: noise decomposition} resulting from the estimation process. 
\begin{restatable}[]{myassump}{defquantizationenoise}
\label{assumption:quantization noise}
For agent $i$, let 
\begin{equation}
    \begin{aligned}
    \label{eq:LA noise decomposition}
      n^{i}_{L}(s,a) &= b^{i}_{L}(s,a)+m^{i}_{L}(s,a) \\
      n^{j,i}_{A}(s,a) &= b^{j,i}_{A}(s,a)+m^{j,i}_{A}(s,a)\quad \forall i,j \neq i
    \end{aligned}
\end{equation}
We assume that the various noise terms $\{b^{j}_{L}(s,a),b^{j,i}_{A}(s,a),m^{j}_{L}(s,a),m^{j,i}_{A}(s,a)\}_{j=1}^{N}$ are independent of $Q$ for any state-action $(s,a)$, independent of any samples collected so far by the agents, and are mutually independent of each other. Furthermore, we assume that the noise terms $\{b^{j}_{L}(s,a),b^{j,i}_{A}(s,a)\}_{j=1}^{N}$ are all sub-Gaussian with mean 0. The terms $\{b^{j}_{L}(s,a)\}$ have a parameter $\sigma^{j}_{L}$, while the terms $\{b^{j,i}_{A}(s,a)\}$ have a parameter $\sigma^{j,i}_{A}$ for all ${i,j}$. 
In addition, all quantization noise terms have mean $0$ and are bounded by the maximal quantization interval, 
\begin{equation*}
    \begin{gathered}
      |m^{i}_{L}(s,a)| \leq \Delta Q_{L}^{i} \quad ;\quad  |m^{j,i}_{L}(s,a)| \leq \Delta Q_{L}^{j,i}\quad \forall i,j \neq i~. 
    \end{gathered}
\end{equation*}
\end{restatable}
In Assumption \ref{assumption:quantization noise}, $b$ represents additive noise and $m$ represents quantization effects, and we allow the parameter of the sub-Gaussian noise to depend on the sending and receiving agents (namely, on the communication channel between agents). We set $\Delta Q_L^i$ to be the quantization bin size for agent $i$ receiving the learners' $Q$-function, and similarly for $\Delta Q^{j,i}_L$.  
\subsection{Additive noise model}
\label{subsec:Additive noise model}
In this subsection we only deal with the additive noise $n^{i}_{L}(s,a) =b^{i}_{L}(s,a),
n^{j,i}_{A}(s,a)=b^{j,i}_{A}(s,a)$, assuming that the quantization terms $m^{i}$ and $m^{j,i}$ in (\eqref{eq: learner and agent noise}-\eqref{eq:LA noise decomposition}) are absent. For each agent there is an inherent tradeoff between using the less noisy $Q$-function received from the learner and the set of individually noisier $Q$-functions from the other agents, that can be potentially mitigated by combining these approximations. The following theorem demonstrates this effect. Each agent estimates the $Q$-function $\hat{Q}^{i}$ as a weighted linear sum
\begin{equation}\label{eq:ApproxSet}
    \begin{gathered}
      \hat{Q}^{i}(s,a) \triangleq w_{ii}\widetilde{Q}^{i}_{L}(s,a) + \sum_{j \neq i}^{N}e_{ji}w_{ji}\widetilde{Q}^{j,i}_{A}(s,a) \qquad ; \qquad \sum_{j=1}^{N}e_{ji}w_{ji}=1~.
    \end{gathered}
\end{equation}
The next theorem provides the agent weighting scheme that minimizes the TCE, thereby optimizing the bound.
\begin{restatable}[]{mytheo}{nonidenticaladditivenoise}
\label{theo: nonidentical additive noise}
Without loss of generality, assume that the agents directly connected to agent $i$ are $\{1,..,d_{i}\}$ not including agent $i$ itself. Under Assumptions \ref{assump:main} and \ref{assumption:quantization noise} with zero quantization noise, at time step $t$ following an update of the $Q$-function $Q$, each agent $i$ estimates the $Q$-function by \eqref{eq:ApproxSet}. Denote
\begin{align*}
         A &\triangleq \sqrt{\boldsymbol{\sigma}_{L}\cdot\boldsymbol{\sigma}_{L}^{T}}+2 \mathrm{diag}\left((\sigma^{1,i}_{A})^{2},...,(\sigma^{d_{i},i}_{A})^{2}\right) +2(\sigma^{i}_{L})^{2}\mathrm{diag}(1,...,1)
        \\
        \mathbf{w} &\triangleq\left(w_{11},...,w_{(d_{i})i}\right)^{T} \quad ; \quad 
        \boldsymbol{\sigma}_{L} \triangleq 2(\sigma^{i}_{L})^{2}\left(1,...,1\right)^{T}
\end{align*}
Then Theorem \ref{theo:sample complexity2} holds. The upper bound is minimized for $\mathbf{w}^{*}=A^{-1}\boldsymbol{\sigma}_{L},w^{*}_{ii}=1-\sum_{j=1}^{d_{i}}w_{ji}^{*}$ and $\sigma^{i*}_{c} = \sigma^{i}_{L}\sqrt{w_{ii}^{*}}$. 
\end{restatable}
\textbf{Proof sketch} A detailed explanation and proofs appear in the Appendix. The idea is to substitute the weight-dependent equivalent noise into the sample complexity bound of Theorem \ref{theo:sample complexity2}, and minimize over the weights. We can show that the optimal weights all have $w^{*}_{ji} \in [0,1]~\forall j$, and therefore we indeed have that $\sigma^{i*}_{c}\leq \sigma^{i}_{L}$, so that it is always worthwhile averaging over different $Q$-functions instead of using only the less noisy $Q$-function from the learner. We also show in the Appendix that the optimal weights of noise terms with a larger parameter $\sigma^{j,i}_{L}$ have a smaller weight, which agrees with intuition. Theorem \ref{theo: identical additive noise}, which is a special case  of Theorem \ref{theo: nonidentical additive noise}, assumes parameters to have one of two possible values,  and thus enables us to see the above-mentioned tradeoff more clearly.
\begin{restatable}[]{mytheo}{identicaladditivenoise}
\label{theo: identical additive noise}
Under the assumptions of Theorem \ref{theo: nonidentical additive noise}, let us further assume that all noise terms for a communication channel have the same parameter $\sigma^{j}_{L}=\sigma_{L},\sigma^{j,i}_{A}=\sigma_{A}$ for all $i$ and $j$, and that the graph $\Gamma$ is $d$-regular (all vertices have exactly $d$ agents). Then: (1) The optimal weights for agent $i$ are
\begin{equation*}
    \begin{gathered}
      w_{ii}^{*} = \frac{\sigma^{2}_{L}+\sigma^{2}_{A}}{(d+1)\sigma^{2}_{L}+\sigma^{2}_{A}}=\frac{1}{1+\frac{d}{1+\sigma^{2}_{A}/\sigma^{2}_{L}}}
      \quad ; \quad 
       w^{*}_{ji} =\frac{1-w_{ii}^{*}}{d}\mathbf{}{1}\{d>0\} \quad \forall j \neq i ~.
    \end{gathered}
\end{equation*}
(2) A uniform weighting of the noisy $Q$-functions is preferable to using $\widetilde{Q}^i_L$ (the received learner's $Q$-function) if and only if $d+1 \geq \sigma^{2}_{A}/\sigma^{2}_{L}$.
\end{restatable}
\noindent From the optimal weights, we can learn that the larger is the fraction $\sigma^{2}_{A}/\sigma^{2}_{L}$ and the smaller is the degree of the graph $d$, the more weight in the averaging will be given to the $Q$-function received from the learner. Similarly, when only a uniform average is possible to the agent, it is worthwhile to do so only if the degree $d$ is large enough compared to the noise parameters ratio. Although it is known that certain communication restrictions, such as collisions, become more dominant with increased connectivity \citep{gallager1985perspective,gupta2000capacity}, we show here that for additive noise the opposite is true, and the performance improves with increased connectivity. 
\begin{figure}
    \includegraphics[width=0.8\textwidth]{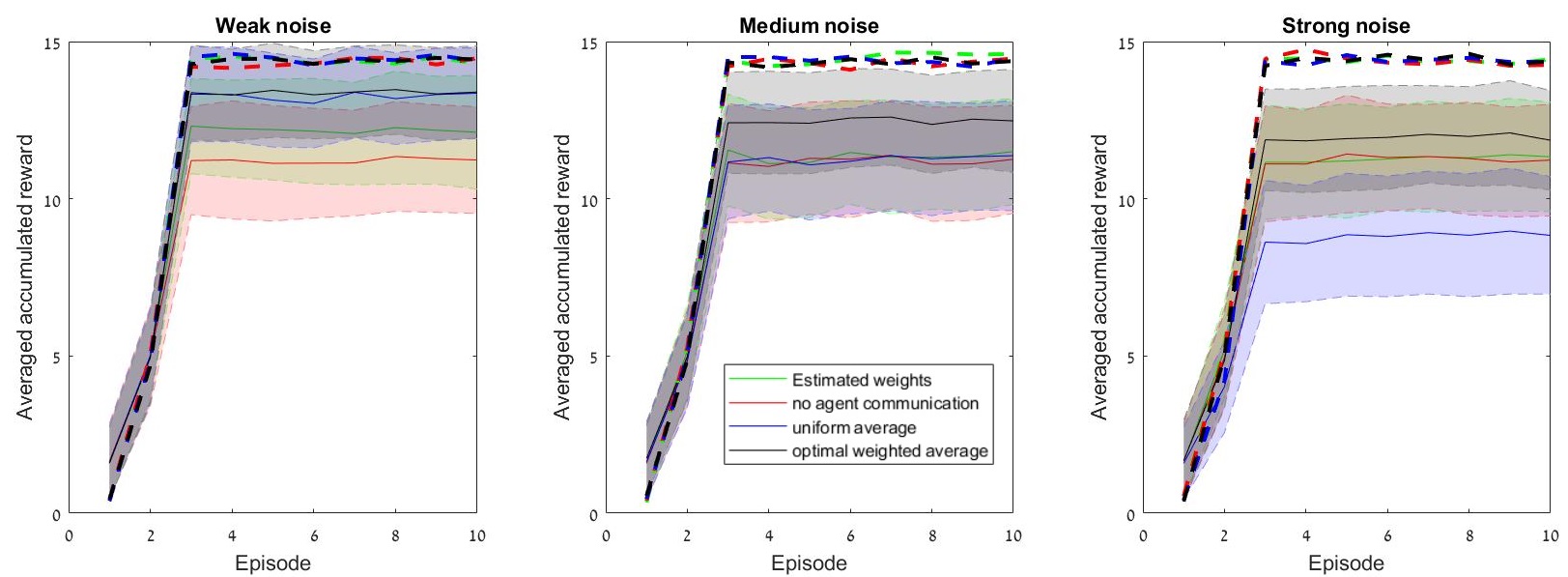}
      \centering
      \caption{\label{fig:simulation}Algorithm \ref{alg:Communication noise tolerant PAC Exploration} in a simple wrap-around $5\times 5$ grid world with $4$ fully-connected agents and additive noise parameters $\sigma^{j}_{L}=\sigma_{L},\sigma^{j,i}_{A}=\sigma_{A},\sigma_{R}=0$ $\forall i,j$. Each graph shows the agent-averaged accumulated reward during an episode, as a function of different episodes (continuous line), and the accumulated reward for a simulated agent with full information regarding the noise-free $Q$-function (dashed line). In green - weighted average with estimated weights, red - using only the $Q$-function sample received from the learner, blue - a uniform average, black - optimal weighted average. Shaded areas are standard deviation.  \textbf{A:} Weak noise $\sigma^{2}_{L}=\sigma^{2}_{A}=0.1$ \textbf{B:} Medium noise $\sigma^{2}_{L}=0.1, \sigma^{2}_{A}=0.4$. \textbf{C:} Strong noise $\sigma^{2}_{L}=0.1, \sigma^{2}_{A}=1$.}
\end{figure}

\noindent \textbf{Numerical demonstration} We compare four weighting schemes. (1) Uniform weights over all received samples. (2) All weight placed on the learners' sample (namely, no inter-agent communication). (3) Optimal weights from Theorem \ref{theo: identical additive noise}. (4) Optimal weights computed empirically (see below). A simple demonstration of the performance of Algorithm \ref{alg:Communication noise tolerant PAC Exploration} under the identical Gaussian additive noise conditions of Theorem \ref{theo: identical additive noise} in a simple $5\times 5$ wrap-around grid world with a goal state, is shown in Figure \ref{fig:simulation} where the cumulative reward is plotted against episode number. The agents communicate via a fully-connected communication graph $\Gamma$, and there is no reward-noise. Further details are given in the Appendix. It is evident that although the exploration is successful in all cases (dashed graphs), the more noisy the communication is, the harder it is for the concurrent agents to get to the goal state. As stated in Theorem \ref{theo: identical additive noise}, it is evident that when the ratio $\sigma^{2}_{A}/\sigma^{2}_L$ is small, a uniform averaging is preferable to using only the less noisy sample, and that the opposite is true when $\sigma^{2}_{A}/\sigma^{2}_L$ increases. Furthermore, the optimal weighting scheme suggested in Theorem \ref{theo: identical additive noise} performs better than both approaches. We also demonstrate the more practical case in which the algorithm has no knowledge about the noise parameters, and instead estimates them on the run, and uses the estimated version for the weights. As can be seen, although this approach is not as good as in the full information case, it performs better on a range of cases compared to using a constant weight.
\subsection{Quantization noise model}
\label{subsec:Quantization noise model}
In Section \ref{subsec:Additive noise model} we assumed that the agents interact by transmitting their noise-corrupted $Q$-function values. Here we add the quantization effects defined in Assumption \ref{assumption:quantization noise}.  
Theorem \ref{theo: quantization noise} presents a bound over the sample complexity in such a case, and an estimation technique for the simpler case of identical quantization levels for each agent.  
\begin{restatable}[]{mytheo}{quannoise}
\label{theo: quantization noise}
Without loss of generality, assume that the agents directly connected to agent $i$ are $\{1,..,d_{i}\}$ excluding agent $i$ itself. Under Assumptions \ref{assumption:quantization noise} and \ref{assump:main}, at time step $t$ following an update of the $Q$-function $Q$, each agent $i$ estimates the $Q$-function $\hat{Q}^{i}$ as a weighted linear sum as in \eqref{eq:ApproxSet}. Then Theorem \ref{theo:sample complexity2} holds with $\Delta Q_{c}^{i} \leq \Delta Q^{i}|w_{ii}|+ \sum_{j = 1}^{d_{i}}\left(\Delta Q^{j}+\Delta Q^{j,i}\right)|w_{ji}|$ and 
\begin{equation*}
	\sigma_{c}(\mathbf{w}) = \sum_{j=1}^{d_{i}}w_{ji}^{2}\left((\sigma^{j}_{L})^{2}+(\sigma^{ji}_{A})^{2}\right)+\left(1-\sum_{j=1}^{d_{i}} w_{ji}\right)^{2}(\sigma^{i}_{L})^{2}~.
\end{equation*}
By further assuming identical agents, $\Delta Q_{L}^{i}=\Delta Q_{L}^{j,i} = \Delta Q,\sigma^{j}_{L}=\sigma^{j,i}_{A}=\sigma_{L}$ for all $i,j$, the optimal weights are 
\begin{equation*}
    \begin{aligned}
         w_{ii}^{*} &= \begin{cases}
               \frac{2}{d+2}\left(1+\frac{d}{\sqrt{(d+2)(2\sigma_{L}^{2}f^{2}/\Delta Q)^{2}-2d}}\right)              & f\sigma_{L} > \Delta Q\\
               1               & f\sigma_{L} \leq \Delta Q\\
         \end{cases}
    \end{aligned}
\end{equation*}
where $w^{*}_{ji} =(1-w^*_{ii})/d, 1 \leq j \in\{1,...,d\}$, and
$f=f(N,k,k_{m},\delta,SA)$ is defined in Theorem \ref{theo:sample complexity2}. 
\end{restatable}
Due to quantization noise, the optimal weight $w^{*}_{ii}$ is larger than that obtained for additive noise alone (which is $2/(d+2)$). Moreover, for quantization noise too large compared to $\sigma_{L}f$, the optimal choice is to use only the less noisy $Q$-function from the learner, i.e., to ignore the other agents. For vanishing quantization noise, we recover Theorem \ref{theo: identical additive noise}. 

Although it is possible to derive an estimation rule for the most general case, we have chosen to assume that the additive noise properties are similar for both the agent-agent noise and the agent-learner noise, in order to stress the tradeoff between additive noise and quantization noise and for simplicity of the resulting expression. 
\section{Closing remarks}
\label{sec:Closing remarks}
We have provided PAC performance guarantees for cooperative MARL under noisy and resource limited communication conditions, and suggested efficient algorithms for that case. By doing so, we have emphasized the tradeoff between the advantages of cross-agent communication and the disadvantages of noisy and restricted communication channels. These results open the door to many future extensions including adaptive learning algorithms for optimal communication protocols, location and agent dependent coordinated exploration, and self-learning by agents. It is also worthwhile to examine the effects of different types of communication restrictions such as latency and collisions that should lead to opposite dependency over the sparsity of the communication network, and to combine them with the effects of additive and quantization noise studied here. 

\vskip 0.2in
\bibliography{Bibliography}

\begin{thebibliography}{39}
\providecommand{\natexlab}[1]{#1}
\providecommand{\url}[1]{\texttt{#1}}
\expandafter\ifx\csname urlstyle\endcsname\relax
  \providecommand{\doi}[1]{doi: #1}\else
  \providecommand{\doi}{doi: \begingroup \urlstyle{rm}\Url}\fi

\bibitem[Akyildiz et~al.(2002)Akyildiz, Su, Sankarasubramaniam, and
  Cayirci]{akyildiz2002wireless}
Ian~F Akyildiz, Weilian Su, Yogesh Sankarasubramaniam, and Erdal Cayirci.
\newblock Wireless sensor networks: a survey.
\newblock \emph{Computer networks}, 38\penalty0 (4):\penalty0 393--422, 2002.

\bibitem[Andrievsky et~al.(2010)Andrievsky, Matveev, and Fradkov]{NoisyControl}
Boris~Rostislavovich Andrievsky, Aleksei~Serafimovich Matveev, and
  Aleksandr~L’vovich Fradkov.
\newblock Control and estimation under information constraints: Toward a
  unified theory of control, computation and communications.
\newblock \emph{Automation and Remote Control}, 71\penalty0 (4):\penalty0
  572--633, 2010.

\bibitem[Bansal et~al.(2017)Bansal, Pachocki, Sidor, Sutskever, and
  Mordatch]{bansal2017emergent}
Trapit Bansal, Jakub Pachocki, Szymon Sidor, Ilya Sutskever, and Igor Mordatch.
\newblock Emergent complexity via multi-agent competition.
\newblock \emph{arXiv preprint arXiv:1710.03748}, 2017.

\bibitem[Baran and Kasal(2010)]{baran2010modeling}
Ond{\v{r}}ej Baran and Miroslav Kasal.
\newblock Modeling of the phase noise in space communication systems.
\newblock \emph{Radioengineering}, 19\penalty0 (1):\penalty0 141--148, 2010.

\bibitem[Bu et~al.(2008)Bu, Babu, De~Schutter, et~al.]{MAReview}
Lucian Bu, Robert Babu, Bart De~Schutter, et~al.
\newblock A comprehensive survey of multiagent reinforcement learning.
\newblock \emph{IEEE Transactions on Systems, Man, and Cybernetics, Part C
  (Applications and Reviews)}, 38\penalty0 (2):\penalty0 156--172, 2008.

\bibitem[Cantelli(1929)]{cantelli1929sui}
Francesco~Paolo Cantelli.
\newblock Sui confini della probabilita.
\newblock In \emph{Atti del Congresso Internazionale dei Matematici: Bologna
  del 3 al 10 de settembre di 1928}, pages 47--60, 1929.

\bibitem[Catoni(2012)]{Catoni}
Olivier Catoni.
\newblock Challenging the empirical mean and empirical variance: a deviation
  study.
\newblock In \emph{Annales de l'IHP Probabilit{\'e}s et statistiques}, pages
  1148--1185, 2012.

\bibitem[Dann et~al.(2017)Dann, Lattimore, and Brunskill]{Dann2017}
Christoph Dann, Tor Lattimore, and Emma Brunskill.
\newblock Unifying pac and regret: Uniform pac bounds for episodic
  reinforcement learning.
\newblock In \emph{Advances in Neural Information Processing Systems}, pages
  5713--5723, 2017.

\bibitem[Devroye et~al.(2016)Devroye, Lerasle, Lugosi, Oliveira,
  et~al.]{Sub-Gaussian}
Luc Devroye, Matthieu Lerasle, Gabor Lugosi, Roberto~I Oliveira, et~al.
\newblock Sub-gaussian mean estimators.
\newblock \emph{The Annals of Statistics}, 44\penalty0 (6):\penalty0
  2695--2725, 2016.

\bibitem[Dimakopoulou and Van~Roy(2018)]{Dimakopoulou2018a}
Maria Dimakopoulou and Benjamin Van~Roy.
\newblock Coordinated exploration in concurrent reinforcement learning.
\newblock \emph{arXiv preprint arXiv:1802.01282}, 2018.

\bibitem[Dimakopoulou et~al.(2018)Dimakopoulou, Osband, and
  Van~Roy]{Dimakopoulou2018}
Maria Dimakopoulou, Ian Osband, and Benjamin Van~Roy.
\newblock Scalable coordinated exploration in concurrent reinforcement
  learning.
\newblock In \emph{Advances in Neural Information Processing Systems}, pages
  4219--4227, 2018.

\bibitem[Doan et~al.(2019)Doan, Maguluri, and Romberg]{Doan2019}
Thinh~T Doan, Siva~Theja Maguluri, and Justin Romberg.
\newblock Convergence rates of distributed td (0) with linear function
  approximation for multi-agent reinforcement learning.
\newblock \emph{arXiv preprint arXiv:1902.07393}, 2019.

\bibitem[Foerster et~al.(2016)Foerster, Assael, de~Freitas, and
  Whiteson]{foerster2016learning}
Jakob Foerster, Ioannis~Alexandros Assael, Nando de~Freitas, and Shimon
  Whiteson.
\newblock Learning to communicate with deep multi-agent reinforcement learning.
\newblock In \emph{Advances in Neural Information Processing Systems}, pages
  2137--2145, 2016.

\bibitem[Gallager(1985)]{gallager1985perspective}
Robert Gallager.
\newblock A perspective on multiaccess channels.
\newblock \emph{IEEE Transactions on information Theory}, 31\penalty0
  (2):\penalty0 124--142, 1985.

\bibitem[Guo and Brunskill(2015)]{Guo2015}
Zhaohan Guo and Emma Brunskill.
\newblock Concurrent pac rl.
\newblock In \emph{Twenty-Ninth AAAI Conference on Artificial Intelligence},
  2015.

\bibitem[Gupta and Kumar(2000)]{gupta2000capacity}
Piyush Gupta and Panganmala~R Kumar.
\newblock The capacity of wireless networks.
\newblock \emph{IEEE Transactions on information theory}, 46\penalty0
  (2):\penalty0 388--404, 2000.

\bibitem[Jiang and Lu(2018)]{Jiang2018}
Jiechuan Jiang and Zongqing Lu.
\newblock Learning attentional communication for multi-agent cooperation.
\newblock In \emph{Advances in Neural Information Processing Systems}, pages
  7254--7264, 2018.

\bibitem[Korda et~al.(2016)Korda, Sz{\"o}r{\'e}nyi, and
  Shuai]{korda2016distributed}
Nathan Korda, Bal{\'a}zs Sz{\"o}r{\'e}nyi, and Li~Shuai.
\newblock Distributed clustering of linear bandits in peer to peer networks.
\newblock In \emph{Journal of machine learning research workshop and conference
  proceedings}, volume~48, pages 1301--1309. International Machine Learning
  Societ, 2016.

\bibitem[Lattimore and Hutter(2014)]{Lattimore2014}
Tor Lattimore and Marcus Hutter.
\newblock Near-optimal pac bounds for discounted mdps.
\newblock \emph{Theoretical Computer Science}, 558:\penalty0 125--143, 2014.

\bibitem[Marco and Neuhoff(2005)]{marco2005validity}
Daniel Marco and David~L Neuhoff.
\newblock The validity of the additive noise model for uniform scalar
  quantizers.
\newblock \emph{IEEE Transactions on Information Theory}, 51\penalty0
  (5):\penalty0 1739--1755, 2005.

\bibitem[Mart{\'\i}nez-Rubio et~al.(2018)Mart{\'\i}nez-Rubio, Kanade, and
  Rebeschini]{martinez2018decentralized}
David Mart{\'\i}nez-Rubio, Varun Kanade, and Patrick Rebeschini.
\newblock Decentralized cooperative stochastic multi-armed bandits.
\newblock \emph{arXiv preprint arXiv:1810.04468}, 2018.

\bibitem[McDiarmid(1989)]{mcdiarmid1989method}
Colin McDiarmid.
\newblock On the method of bounded differences.
\newblock \emph{Surveys in combinatorics}, 141\penalty0 (1):\penalty0 148--188,
  1989.

\bibitem[Mordatch and Abbeel(2018)]{Mordatch2017}
Igor Mordatch and Pieter Abbeel.
\newblock Emergence of grounded compositional language in multi-agent
  populations.
\newblock In \emph{Thirty-Second AAAI Conference on Artificial Intelligence},
  2018.

\bibitem[Norwood et~al.(1977)Norwood, Hinkelmann, et~al.]{VarEst}
Thomas~E Norwood, Klaus Hinkelmann, et~al.
\newblock Estimating the common mean of several normal populations.
\newblock \emph{The Annals of Statistics}, 5\penalty0 (5):\penalty0 1047--1050,
  1977.

\bibitem[Omidshafiei et~al.(2017)Omidshafiei, Pazis, Amato, How, and
  Vian]{omidshafiei2017deep}
Shayegan Omidshafiei, Jason Pazis, Christopher Amato, Jonathan~P How, and John
  Vian.
\newblock Deep decentralized multi-task multi-agent reinforcement learning
  under partial observability.
\newblock In \emph{Proceedings of the 34th International Conference on Machine
  Learning-Volume 70}, pages 2681--2690. JMLR. org, 2017.

\bibitem[Osband et~al.(2014)Osband, Van~Roy, and Wen]{Osband2014}
Ian Osband, Benjamin Van~Roy, and Zheng Wen.
\newblock Generalization and exploration via randomized value functions.
\newblock \emph{arXiv preprint arXiv:1402.0635}, 2014.

\bibitem[Panait and Luke(2005)]{panait2005cooperative}
Liviu Panait and Sean Luke.
\newblock Cooperative multi-agent learning: The state of the art.
\newblock \emph{Autonomous agents and multi-agent systems}, 11\penalty0
  (3):\penalty0 387--434, 2005.

\bibitem[Parisotto et~al.(2019)Parisotto, Ghosh, Yalamanchi, Chinnaobireddy,
  Wu, and Salakhutdinov]{Parisotto2019}
Emilio Parisotto, Soham Ghosh, Sai~Bhargav Yalamanchi, Varsha Chinnaobireddy,
  Yuhuai Wu, and Ruslan Salakhutdinov.
\newblock Concurrent meta reinforcement learning.
\newblock \emph{arXiv preprint arXiv:1903.02710}, 2019.

\bibitem[Pazis and Parr(2016)]{Pazis2016a}
Jason Pazis and Ronald Parr.
\newblock Efficient pac-optimal exploration in concurrent, continuous state
  mdps with delayed updates.
\newblock In \emph{Thirtieth AAAI Conference on Artificial Intelligence}, 2016.

\bibitem[Pazis et~al.(2016)Pazis, Parr, and How]{Pazis2016}
Jason Pazis, Ronald~E Parr, and Jonathan~P How.
\newblock Improving pac exploration using the median of means.
\newblock In \emph{Advances in Neural Information Processing Systems}, pages
  3898--3906, 2016.

\bibitem[Robinson and Spector(2002)]{robinson2002using}
Alan~J Robinson and Lee Spector.
\newblock Using genetic programming with multiple data types and automatic
  modularization to evolve decentralized and coordinated navigation in
  multi-agent systems.
\newblock In \emph{GECCO Late Breaking Papers}, pages 391--396, 2002.

\bibitem[Roth et~al.(2007)Roth, Simmons, and Veloso]{roth2007execution}
Maayan Roth, Reid Simmons, and Manuela Veloso.
\newblock \emph{Execution-time communication decisions for coordination of
  multi-agent teams}.
\newblock PhD thesis, Carnegie Mellon University, The Robotics Institute, 2007.

\bibitem[Schott and Harville(2006)]{detlemma}
James~R. Schott and David~A. Harville.
\newblock {Matrix Algebra from a Statistician's Perspective}.
\newblock \emph{Journal of the American Statistical Association}, 93\penalty0
  (443):\penalty0 1236, 2006.
\newblock ISSN 01621459.
\newblock \doi{10.2307/2669871}.

\bibitem[Silver et~al.(2018)Silver, Hubert, Schrittwieser, Antonoglou, Lai,
  Guez, Lanctot, Sifre, Kumaran, Graepel, et~al.]{silver2018general}
David Silver, Thomas Hubert, Julian Schrittwieser, Ioannis Antonoglou, Matthew
  Lai, Arthur Guez, Marc Lanctot, Laurent Sifre, Dharshan Kumaran, Thore
  Graepel, et~al.
\newblock A general reinforcement learning algorithm that masters chess, shogi,
  and go through self-play.
\newblock \emph{Science}, 362\penalty0 (6419):\penalty0 1140--1144, 2018.

\bibitem[Srinivasan and Choy(2006)]{srinivasan2006cooperative}
D~Srinivasan and MC~Choy.
\newblock Cooperative multi-agent system for coordinated traffic signal
  control.
\newblock In \emph{IEE Proceedings-Intelligent Transport Systems}, volume 153,
  pages 41--50. IET, 2006.

\bibitem[Stone and Veloso(2000)]{stone2000multiagent}
Peter Stone and Manuela Veloso.
\newblock Multiagent systems: A survey from a machine learning perspective.
\newblock \emph{Autonomous Robots}, 8\penalty0 (3):\penalty0 345--383, 2000.

\bibitem[Sukhbaatar et~al.(2016)Sukhbaatar, Fergus, et~al.]{Sukhbaatar2016}
Sainbayar Sukhbaatar, Rob Fergus, et~al.
\newblock Learning multiagent communication with backpropagation.
\newblock In \emph{Advances in Neural Information Processing Systems}, pages
  2244--2252, 2016.

\bibitem[Viseras et~al.(2018)Viseras, Xu, and Merino]{viseras2018distributed}
Alberto Viseras, Zhe Xu, and Luis Merino.
\newblock Distributed multi-robot cooperation for information gathering under
  communication constraints.
\newblock In \emph{2018 IEEE International Conference on Robotics and
  Automation (ICRA)}, pages 1267--1272. IEEE, 2018.

\bibitem[Walden(1999)]{walden1999analog}
Robert~H Walden.
\newblock Analog-to-digital converter survey and analysis.
\newblock \emph{IEEE Journal on selected areas in communications}, 17\penalty0
  (4):\penalty0 539--550, 1999.

\end{thebibliography}
\newpage

\appendix
\section{Efficient PAC proofs}
\label{sec:Appendix A: Efficient PAC proofs}
In this section we provide proofs for the Efficient-PAC theorems presented in Section \ref{sec:Communication noise exploration bounds}. We also note that in this work we use the median-of-means estimate for the sample-mean (Definition \ref{def: F operator}), instead of using the more standard empirical-mean estimate. It was recently demonstrated that the former has improved convergence results over the latter in the absence of information about the underlying distribution \citep{Pazis2016, Catoni, Sub-Gaussian}.
\subsection{Definitions}
\label{subsec:Definitions}
First, we define the sample set as a collection of previously agent-collected samples. We note that only the learner has access to the full sample set, which is used for value iteration. 
\begin{restatable}[]{mydef}{sample_set}
\label{def:sample_set}
For a state-action $(s,a)$, the sample set $u(s,a)$ is a set of up to $k$ tuple samples $x_{l}\triangleq (s,a,\widetilde{R}_{n(l),t(l)},s_{n(l),t(l)+1})$, where $l$ is the sample's index, $n(l)$ is the collecting agent and $t(l)$ is the time at which the sample was collected. $\widetilde{R}_{n(l),t(l)}$ is the corresponding noisy reward, and $s_{n(l),t(l)+1}$ is the value of the next state. 
\end{restatable}
For simplicity of representation, we define the saturated $Q$-function to be the restriction of some $Q$-function to the set $[0,Q_{max}]$ for each state-action, for some given $Q$ and $Q_{max}$.
\begin{restatable}[]{mydef}{Q^{sat}}
\label{def:Qsat}
Given a function $Q:S\times A \rightarrow \mathbb{R} $ and a positive value $Q_{max}$, the \emph{saturated $Q$-function} is, 
\begin{equation*}
    Q_{\sat}(s,a)=\max\{0,\min\{Q_{\max},Q(s,a))\}\}
\end{equation*}
\end{restatable}
The next operator $F$ is used to compute the approximate Bellman operator in Algorithm \ref{alg:Communication noise tolerant PAC Exploration}, and is essentially an empiric estimate of the Bellman operator with an additive Upper Confidence Bound (UCB) term (as stated before, we use the median-of-means estimate for the sample mean corresponding to $R(s,a,s')+\gamma Q(s',\pi(s'))$ instead of the empirical average). For given integers $k,k_{m}$ such that $k/k_{m}$ is some power of $2$, given that the number of samples in the sample set $u(s,a)$ is $2^{p}k_{m}$ with some integer $p$, we divide the group arbitrarily to $k_{m}$ groups with $2^{p}$ samples each, calculate the sample average resulting from this group as the operator $G$, and then define $F$ to be the median of the various $G$ operators, with an additive UCB term. 
\begin{restatable}[]{mydef}{Foperator}
\label{def: F operator}
Let $k_{m} \geq 1$ be an integer parameter, $\pi$ some fixed policy, $\epsilon_{b}$ a real value and $u(s,a)$ be a sample set with $|u(s,a)| \triangleq 2^{p}k_{m}$ samples for some integer $p \geq 0$. The operator $F^{\pi}(Q,u(s,a))$ is defined as
\begin{equation}
\label{eq:DefineF}
    \begin{gathered}
            F^{\pi}(Q,u(s,a)) \triangleq \frac{\epsilon_{b}}{\sqrt{|u(s,a)|}}+\mathrm{median}\{G^{\pi}(Q,u(s,a),1),...,G^{\pi}(Q,u(s,a),k_{m})\}~, 
    \end{gathered}
\end{equation}
where
\begin{equation}
\label{eq:DefineG}
    \begin{gathered}
            G^{\pi}(Q,u(s,a),j) \triangleq \frac{k_{m}}{|u(s,a)|}\sum^{j\frac{|u(s,a)|}{k_{m}}}_{l=1+(j-1)\frac{|u(s,a)|}{k_{m}}}\left(\widetilde{R}_{n(l),t(l)}+\gamma Q(s_{n(l),t(l)+1},\pi(s_{n(l),t(l)+1}))\right)~,  
    \end{gathered}
\end{equation}
and $(\widetilde{R}_{n(l),t(l)},s_{n(l),t(l)+1})$ is the $l$-th noisy sample in $u(s,a)$ collected by agent $n(l)$ at time $t(l)$. We will use $\widetilde{F}$ to denote $\widetilde{F}^{\pi^{Q}}$, where $\pi^{Q}$ is the greedy policy over $Q$.
\end{restatable}
The approximate Bellman operator $\widetilde{B}$ from Algorithm \ref{alg:Communication noise tolerant PAC Exploration} is defined as a mere saturation.
\begin{restatable}[]{mydef}{Boperator}
\label{def: B operator}
For a state-action $(s,a)$, the approximate Bellman operator $\widetilde{B}^{\pi}$  for policy ${\pi}$ is 
\begin{equation*}
    \begin{gathered}
            \widetilde{B}^{\pi}Q(s,a) \triangleq \max\{0,\min\{Q_{\max},F^{\pi}(Q,u(s,a))\}\}
    \end{gathered}
\end{equation*}
We will use $\widetilde{B}$ to denote $\widetilde{B}^{\pi^{Q}}$. When $|u(s,a)|=0$ we set $\widetilde{B}^{\pi}Q(s,a) \triangleq Q_{\max}$.
\end{restatable}
The variance of the Bellman operator $\sigma$ is defined to an upper bound on the possible values of the variance for all terms that can be encountered by the algorithm. 
\begin{restatable}[]{mydef}{sig} \label{def: sigma}
The variance of the Bellman operator $\sigma$ is defined to be the minimal constant satisfying
\begin{equation*}
    \sum_{s^{'}}p(s'|s,a)\left(R(s,a,s')+\gamma  Q(s^{'},\pi^{Q}(s^{'}))-B^{\pi^{Q}}Q(s,a)\right)^{2} \leq \sigma^{2} \quad \forall (s,a,\pi^{Q},Q),
\end{equation*}
where $Q$ can be any $Q$-function produced by the learner during the run of the algorithm, and $\pi^{Q}$ is the corresponding greedy policy. 
\end{restatable}
An efficient exploration algorithm is defined as an algorithm for which the computational complexity, the space complexity and the sample complexity are bounded. Intuitively, the sample complexity can be thought of as the number of samples that need to be collected before the algorithm achieves the exploratory goal, and is commonly defined as the number of sub-optimal steps in which the algorithm is `far' from its goal. We define a measure for the sample complexity termed the \emph{Total Cost of Exploration} (TCE) which is shown to dominate the number of sub-optimal steps measure \citep{Pazis2016a}. This measure takes into account how sub-optimal the steps are by summing the difference between the value function under the policy produced by the algorithm, and the value of the optimal policy. It can be easily shown that if the TCE of an algorithm is bounded by some term $L$, then the number of sub-optimal steps measure is bounded by $L/\epsilon$.
\begin{restatable}[]{mydef}{TCE}
\label{def:TCE}
Let $\pi$ be a possibly non-stationary and history dependent policy, $(s_{1},s_{2},...)$ a random path generated by $\pi$, $\epsilon$ a positive constant, $T$ the (possibly infinite) set of time steps for which $V^{\pi}(s_{t})<V^{*}(s_{t})-\epsilon$, and define
\begin{equation*}
    \begin{aligned}
            \epsilon_{e}(t) &=V^{*}(s_{t})-V^{\pi}(s_{t})-\epsilon,\quad \forall t \in T\\
            \epsilon_{e}(t) &=0,\quad \forall t \notin T
    \end{aligned}
\end{equation*}
The Total Cost of Exploration (TCE) for parameter $\epsilon$ is defined as the un-discounted infinite sum
$TCE_{\epsilon} \triangleq\sum_{t=0}^{\infty}\epsilon_{e}(t)$. To avoid encumbering the notation, we use $TCE$ instead of $TCE_{\epsilon}$ when the context is clear. 
\end{restatable}
A single-agent exploration algorithm is defined to be Probably Approximately Correct (PAC) if its sample complexity, computational complexity and space complexity are all bounded by some polynomial.
\begin{restatable}[]{mydef}{PAC}
\label{def:PAC}
A single agent algorithm is said to be efficient PAC-MDP, if for any $\varepsilon>0$ and $0<\delta<1$, its sample complexity, its per-time step computational complexity, and its space complexity, are less than some polynomial in the relevant quantities $(S,A,1/\varepsilon,1/\delta,1/1-\gamma)$, with probability of at least $1-\delta$. 
\end{restatable}
We will define a multi-agent algorithm to be efficient PAC-MDP, if Definition \ref{def:PAC} holds with regards to a sample complexity that is defined as the sum of sample complexities for all agents. By using the TCE measure, the TCE of a multi agent system is thus TCE$=\sum^{N}_{i=1}$TCE$(i)$ where $TCE(i)$ is the TCE for agent $i$ for $\epsilon$.
\begin{restatable}[]{mydef}{MAPAC}
\label{def:MAPAC}
A multi-agent algorithm is said to be efficient PAC-MDP, if for any $\varepsilon>0$ and $0<\delta<1$, its sample complexity TCE$=\sum^{N}_{i=1}$TCE$(i)$, per-time step computational complexity, and space complexity, are all less than some polynomial in the relevant quantities $(S,A,1/\varepsilon,1/\delta,1/1-\gamma)$, with probability of at least $1-\delta$. The sample complexity is the sum of sample complexities of all agents.  
\end{restatable}
A flow chart for the dependencies between the various lemmas and theorems we state here can be found in Figure \ref{fig:proof flow}.
\begin{figure}
    \includegraphics[width=0.9\textwidth]{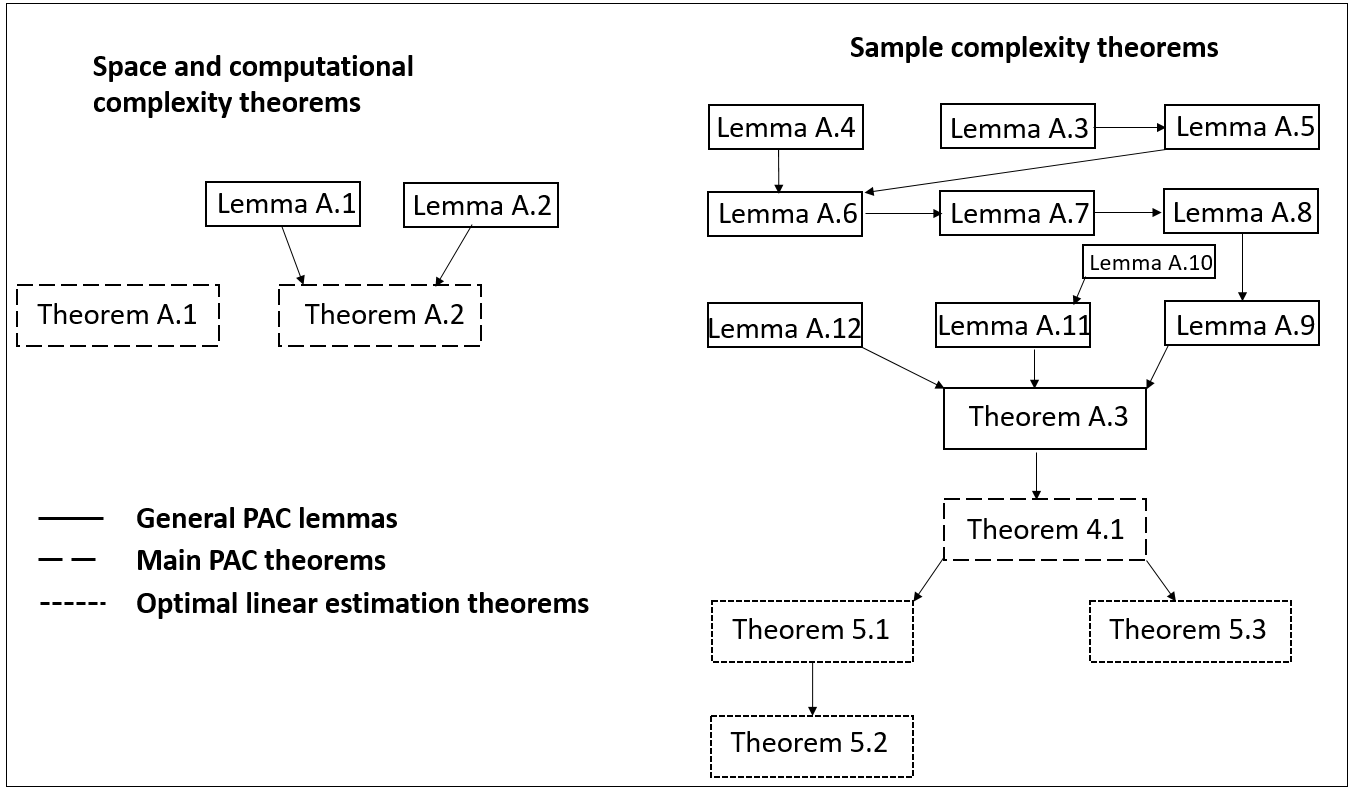}
      \centering
      \caption{\label{fig:proof flow}Dependencies between the various lemmas and proofs stated in this paper. Continuous black rectangles stand for general PAC lemmas, dashed squares indicate the main PAC theorems, and short-dash squares stand for theorems that use the optimal linear estimation methods of Section \ref{sec:Estimation methods}.}
\end{figure}
\subsection{Space and computational complexity bounds}
\label{subsec:Space and computational complexity bounds}
Since Algorithm \ref{alg: Advance N agents} is a sub-algorithm for Algorithm \ref{alg:Communication noise tolerant PAC Exploration}, in the following theorems we only state results for the latter.
\begin{restatable}[]{mytheo}{spacecomplexity}
\label{theo:space complexity}
The space complexity of Algorithm \ref{alg:Communication noise tolerant PAC Exploration} is $O\left((k+N)SA\right)$.
\end{restatable}
\begin{proof}
Each agent in Algorithm \ref{alg:Communication noise tolerant PAC Exploration} only needs access to the Q-function, and the estimated Q-function is calculated as a weighted average of $d_{i}+1$ samples (from the learner and from the other agents). Thus, only one value per state-action  has to be saved, leading to a term of $O\left(NSA\right)$ for all $N$ agents. The learner keeps a copy of the $Q$-function, and for each state-action saves up to $O(k)$ samples. This results in a term of $O\left(kSA\right)$. Overall, the space complexity is $O\left((k+N)SA\right)$
\end{proof}
In order to bound the computational complexity of the algorithm, we first quote Lemma 13.3 from \citep{Pazis2016a} which bounds the number of value iterations needed in order to converge to a fixed point given some confidence level. 
\begin{restatable}[]{mylemma}{B_conv}
\label{lemma:B_conv}
Let $\hat{B}$ be a $\gamma$-contraction with fixed point $\hat{Q}$, and $Q$ the output of $(1/(1-\gamma))\ln Q_{max}/\epsilon$
iterations of value iteration using $\hat{B}$. Then if $Q,Q_{0} \in [0,Q_{max}]$ where $Q_{0}$ is the initial Q-function, then
\begin{equation*}
    |Q(s,a)-\hat{B}Q(s,a)|\leq \epsilon \quad \forall (s,a)
\end{equation*}
\end{restatable}
We now show that the operator $\widetilde{B}$ is indeed a $\gamma$-contraction.
\begin{restatable}[]{mylemma}{contraction}
\label{lemma:contraction}
$\widetilde{B}$ is a $\gamma$-contraction under the maximum norm.
\end{restatable}
\begin{proof}
given some real parameter $x$, note that if we define the saturation function 
\begin{equation*}
    f_{sat}(x)= \max\{f_{min},\min\{f_{max},x\}\}
\end{equation*}
for some $f_{max} \geq f_{min}$, then given some $x_{0}\geq 0$, we have that $f_{sat}(x+x_{0}) \leq f_{sat}(x)+x_{0}$. It also holds that $f_{sat}$ is non-decreasing. Now define $||Q_{1}-Q_{2}||_{\infty}=\epsilon$. To prove the lemma, we must show that $||\widetilde{B}Q_{1}-\widetilde{B}Q_{2}||_{\infty}\leq \gamma\epsilon$. 

First, note that any function $G(Q,u(s,a,),j)$ (see Definition \ref{def: F operator}) for $1 \leq j\leq k_{m}$ is a simple average of terms containing maximization over the Q-function, and so under the definition of $\epsilon$ above
\begin{equation*}
    G(Q_{1},u(s,a),j) \leq G(Q_{2},u(s,a),j) + \gamma\epsilon
\end{equation*}
For all state-actions and values of $j$. Since the operator F from Definition \ref{def: F operator} contains a mere median function over the $G$ functions, the following also holds 
\begin{equation*}
    F(Q_{1},u(s,a)) \leq F(Q_{2},u(s,a)) + \gamma\epsilon
\end{equation*}
And now for the approximate Bellman operator we have 
\begin{equation*}
    \widetilde{B}Q_{1}(s,a) = f_{sat}(F(Q_{1},u(s,a))) \leq f_{sat}(F(Q_{2},u(s,a)))+\gamma\epsilon=\widetilde{B}Q_{2}(s,a)+\gamma\epsilon
\end{equation*}
Similarly we can show that $\widetilde{B}Q_{2}(s,a) \leq \widetilde{B}Q_{1}(s,a)+\gamma\epsilon$, which proves the theorem.
\end{proof}
We use the previous results to prove a bound over the computational complexity of the algorithm. 
\begin{restatable}[]{mytheo}{computationalcomplexity}
\label{theo:computational complexity}
The per step computational complexity of Algorithm \ref{alg:Communication noise tolerant PAC Exploration} is
\begin{equation*}
    \label{eq:computational complexity}
    O\left(NSAd_{max}+\frac{kSA^{2}}{1-\gamma}\ln\frac{Q_{max}}{\epsilon_{a}}\right)
\end{equation*}
where $d_{max}=\max_{i}d_{i}$.
\end{restatable}
\begin{proof}
Lines 3-12 in Algorithm \ref{alg: Advance N agents} are performed once for each of the $N$ agents. For each agent, we first calculate a weighted average over $d_{i}$ samples, and this is done $SA$ times - once per every state-action. We then use a greedy policy that requires a search over $A$ actions. This results in Algorithm \ref{alg: Advance N agents} having a computational complexity bounded by $O\left(NSAd_{max}\right)$ where $d_{max}\triangleq \max_{i}d_{i}$.

Next, we have that lines 5-12 in Algorithm \ref{alg:Communication noise tolerant PAC Exploration} are also performed $N$ times, once for each sample collected by an agent. Assuming that nullifying an array of samples can be performed in some constant time independent of the number of elements, we have a contribution of $O(N)$ to the computational complexity from lines 6-12. Since the operator $\widetilde{B}$ is a $\gamma$-contraction and since the $Q$-function satisfies $Q \in [0,Q_{max}]$ at all times by the definition of $\widetilde{B}$, we can use Lemma \ref{lemma:B_conv} and deduce that the value iteration in lines 15-17 will be performed at most $\frac{1}{1-\gamma}\ln\frac{Q_{max}}{\epsilon_{a}}$ times for each state-action. By definitions \ref{def: F operator},\ref{def: B operator}, A single execution of $\widetilde{B}$ for a given state-action requires $k_{m}$ calculations of an average at the cost of $\frac{k}{k_{m}}A$ operations each, and is thus performed $kA$ times. 

Thus, the overall computational complexity is
\begin{equation*}
O\left(NSAd_{max}+\frac{kSA^{2}}{1-\gamma}\ln\frac{Q_{max}}{\epsilon_{a}}\right)~. 
\end{equation*}
\end{proof}


\subsection{Sample complexity bounds}
\label{subsec:Sample complexity bounds}
We now introduce various lemmas that will help us prove a sample complexity bound in Theorem \ref{theo:sample complexity}, which will be used to prove Theorem \ref{theo:sample complexity2}. 
First, for completeness, we quote two concentration inequalities that we use in theorems \ref{lemma:GtoB} and \ref{lemma:F_lemma}. The first one is Cantelli's inequality \citep{cantelli1929sui}.
\begin{restatable}[]{mylemma}{Cantelli}
\label{lemma:Cantelli}
Let $X$ be a real valued random variable with a finite variance $\sigma^{2}$. Then for any $\lambda > 0 $
\begin{equation*}
    P(X-EX \geq \lambda) \leq \frac{\sigma^{2}}{\sigma^{2}+\lambda^{2}},\quad;\quad 
    P(X-EX \leq -\lambda) \leq \frac{\sigma^{2}}{\sigma^{2}+\lambda^{2}},
\end{equation*}
\end{restatable}
The second one is McDiarmid's inequality \citep{mcdiarmid1989method}. 
\begin{restatable}[]{mylemma}{Mcdiarmid}
\label{lemma:Mcdiarmid}
Let $X_{1},...,X_{n} \in \mathbb{M}$ be independent random variables, and $f$ be a mapping $f:\mathbb{M}^{n} \rightarrow \mathbb{R}$. If, for all $i \in \{1,...,n\}$, and for all $x_{1},...,x_{n},x_{i^{'}} \in \mathbb{M}$, the function $f$ satisfies
\begin{equation*}
    |f(x_{1},...,x_{i-1},x_{i},x_{i+1},...,x_{n})-f(x_{1},...,x_{i-1},x_{i^{'}},x_{i+1},...,x_{n})| \leq c_{i}~, 
\end{equation*}
for some positive constants $c_{1},...,c_{n}$, then $\forall t \geq 0$
\begin{equation*}
    P(f(x_{1},...,x_{n})-Ef(x_{1},...,x_{n}) \geq t) \leq \exp \left(\frac{-2t^{2}}{\sum_{i=1}^{n}c_{i}^{2}}\right)
\end{equation*}
and 
\begin{equation*}
    P(f(x_{1},...,x_{n})-Ef(x_{1},...,x_{n}) \leq -t) \leq \exp \left(\frac{-2t^{2}}{\sum_{i=1}^{n}c_{i}^{2}}\right)
\end{equation*}
\end{restatable}
In lemma \ref{lemma:GtoB} we show that the function $G$ from Definition \ref{def: F operator} is close to the true Bellman operator with some probability.  
\begin{restatable}[]{mylemma}{GtoB}
\label{lemma:GtoB}
Let $\sigma$ be defined as in Definition \ref{def: sigma}, $\epsilon_{b}^{2}=4k_{m}(\sigma^{2}+\sigma^{2}_{R})$, and $k_{m}=\ceil{5.6\ln\frac{8N\ceil{1+\log_{2}\frac{k}{k_{m}}}(SA)^{2}}{\delta}}$. Under assumption \ref{assump: Noise_assumptions1} given a fixed $Q$, a fixed state action $(s,a)$ and a fixed $j \in \{1,...,k_{m}\}$ we have
\begin{equation*}
   P\left(G^{\pi}(Q,u(s,a),j)-B^{\pi}Q(s,a) \leq -\frac{\epsilon_{b}}{\sqrt{|u(s,a)|}}\right) \leq \frac{1}{5}
\end{equation*}
and
\begin{equation*}
\begin{gathered}
   P\left(G^{\pi}(Q,u(s,a),j)-B^{\pi}Q(s,a) \geq \frac{\epsilon_{b}}{\sqrt{|u(s,a)|}}\right) \leq \frac{1}{5}
   \end{gathered}
\end{equation*}
\end{restatable}
\begin{proof}
Using the definition of $G$ \eqref{eq:DefineG} we have 
\begin{equation*}
    \begin{gathered}
            G^{\pi}(Q,u(s,a),j) = \frac{k_{m}}{|u(s,a)|}\sum^{j\frac{|u(s,a)|}{k_{m}}}_{l=1+(j-1)\frac{|u(s,a)|}{k_{m}}}\left(R_{l}+n_{R,l}+\gamma Q(s^{'}_{l},\pi(s^{'}_{l}))\right) 
    \end{gathered}
\end{equation*}
Where for convenience we denote $R_{l}$ to be the actual reward of sample $l$ in this group (which is independent of the collecting agent), $n_{R,l}$ to be the reward noise of sample $l$ collected by agent $n(l)$, and $s^{'}_{l}$ to be the next state value for sample $l$. Since the reward-noise is zero-mean, by taking the expectation operator we have
\begin{equation*}
    \begin{gathered}
            EG^{\pi}(Q,u(s,a),j) = B^{\pi}Q(s,a)
    \end{gathered}
\end{equation*}
Now, let us denote by $X_{1},...,X_{\frac{|u(s,a)|}{k_{m}}}$ the random variables such that $X_{l}=R_{l}+\gamma Q(s^{'}_{l},\pi(s^{'}_{l}))$, and by $Y_{1},...,Y_{\frac{|u(s,a)|}{k_{m}}}$ the random variables such that $Y_{l}=n_{R,l}$. Since we assume a given state action $(s,a)$ and index $j$, a given $Q$-function and use the independence assumptions of assumption \ref{assump: Noise_assumptions1}, we have that the random variables $X_{1},...,X_{\frac{|u(s,a)|}{k_{m}}},Y_{1},...,Y_{\frac{|u(s,a)|}{k_{m}}}$ are all mutually independent. To understand this, note that since we fixed the values of $(s,a),j,Q$ - the randomness defining these random varibales depends only on the next-state values $s'_{l}$ and on the noise values $n_{R,l}$. Both of these terms are mutually independent given (s,a) from the markov property. 

By Definition \ref{def: sigma} and assumption \ref{assump: Noise_assumptions1} defining $\sigma$ and $\sigma_{R}$, it is also true that $VarX_{l} \leq \sigma^{2}$ and $VarY_{l} \leq \sigma^{2}_{R}$ for any sample $l$, and from the independence property $VarG \leq \frac{k_{m}}{|u(s,a)|}(\sigma^{2}+\sigma_{R}^{2})$.
Now using Cantelli's inequality \ref{lemma:Cantelli} 
\begin{equation*}
\begin{aligned}
          P\left(G^{\pi}(Q,u(s,a),j)-B^{\pi}Q(s,a) \leq -\frac{\epsilon_{b}}{\sqrt{|u(s,a)|}}\right)  &=  \\ P\left(G^{\pi}(Q,u(s,a),j)-EG^{\pi}(Q,u(s,a),j) \leq -\frac{\epsilon_{b}}{\sqrt{|u(s,a)|}}\right) &\leq \\ \frac{VarG}{VarG+\epsilon_{b}^{2}/|u(s,a)|} \leq 
          \frac{\frac{k_{m}}{|u(s,a)|}(\sigma^{2}+\sigma_{R}^{2})}{\frac{k_{m}}{|u(s,a)|}(\sigma^{2}+\sigma_{R}^{2})+4\frac{k_{m}}{|u(s,a)|}(\sigma^{2}+\sigma_{R}^{2})}&=\frac{1}{5}
\end{aligned}
\end{equation*}
We prove the other inequality by the same argument.
\end{proof}
\begin{restatable}[]{mylemma}{F_lemma}
\label{lemma:F_lemma}
Let $\sigma$ be defined as in Definition \ref{def: sigma}, $\epsilon_{b}$, and $k_{m}$ as in Lemma \ref{lemma:GtoB}.  Under assumption \ref{assump: Noise_assumptions1}, given a fixed $Q$ and a fixed state action $(s,a)$ we have
\begin{equation*}
   P\left(F^{\pi^{*}}(Q,u(s,a))-B^{\pi^{*}}Q(s,a) \leq 0 \right) \leq \frac{\delta}{8N\ceil{1+\log_{2}\frac{k}{k_{m}}}(SA)^{2}}
\end{equation*}
and
\begin{equation*}
\begin{gathered}
   P\left(F^{\pi^{Q}}(Q,u(s,a))-B^{\pi^{Q}}Q(s,a) \geq 2\frac{\epsilon_{b}}{\sqrt{|u(s,a)|}}\right) \leq  \frac{\delta}{8N\ceil{1+\log_{2}\frac{k}{k_{m}}}(SA)^{2}}
   \end{gathered}
\end{equation*}
\end{restatable}
\begin{proof}
Define $Z_{1},...,Z_{k_{m}}$ as random variables such that $Z_{j}$ represents the joint distribution of all the samples in the $j$-th group used by the function $G^{\pi}(Q,u(s,a),j)$. We now define the following two counter functions. $f^{+}$ counts the number of groups in which $G^{\pi}(Q,u(s,a),j)-B^{\pi}Q(s,a) \leq -\frac{\epsilon_{b}}{\sqrt{|u(s,a)|}}$
\begin{equation*}
    f^{+}(Z_{1},...,Z_{k_{m}})=\sum^{k_{m}}_{j=1}\mathbf{1}\{G^{\pi}(Q,u(s,a),j)-B^{\pi}Q(s,a) \leq -\frac{\epsilon_{b}}{\sqrt{|u(s,a)|}}\}
\end{equation*}
Where $\mathbf{1}\{\dot\}$ is the indicator function, and $f^{-}$ counts the number of groups in which $G^{\pi}(Q,u(s,a),j)-B^{\pi}Q(s,a) \geq \frac{\epsilon_{b}}{\sqrt{|u(s,a)|}}$
\begin{equation*}
    f^{-}(Z_{1},...,Z_{k_{m}})=\sum^{k_{m}}_{j=1}\mathbf{1}\{G^{\pi}(Q,u(s,a),j)-B^{\pi}Q(s,a) \geq \frac{\epsilon_{b}}{\sqrt{|u(s,a)|}}\}
\end{equation*}
By lemma \ref{lemma:GtoB} We have that 
\begin{equation*}
    \begin{aligned}
            Ef^{+} \leq \frac{k_{m}}{5}\\
            Ef^{-} \leq \frac{k_{m}}{5}
    \end{aligned}
\end{equation*}
Since $f^{+},f^{-}$ count events, it is obvious that a change in the value of one random variable $Z_{i}$ results in a bounded change of these functions. 
\begin{equation*}
\begin{aligned}
        |f^{+}(Z_{1},...,Z_{i-1},Z_{i},Z_{i+1},...,Z_{k_{m}})-f^{+}(Z_{1},...,Z_{i-1},Z_{i^{'}},Z_{i+1},...,Z_{k_{m}})| \leq 1 \\ 
        |f^{-}(Z_{1},...,Z_{i-1},Z_{i},Z_{i+1},...,Z_{k_{m}})-f^{-}(Z_{1},...,Z_{i-1},Z_{i^{'}},Z_{i+1},...,Z_{k_{m}})| \leq 1
\end{aligned}
\end{equation*}
By using the same arguments as in lemma \ref{lemma:GtoB} the variables $Z_{1},...,Z_{k_{m}}$ are independent given $(s,a)$ and the $Q$-function, so we can use Mcdiarmid's inequality \ref{lemma:Mcdiarmid}.
\begin{equation*}
    \begin{aligned}
            &P\left(\mathrm{median}\{G^{\pi^{*}}(Q,u(s,a),1),...,G^{\pi^{*}}(Q,u(s,a),k_{m})\}-B^{\pi^{*}}Q(s,a) \leq -\frac{\epsilon_{b}}{\sqrt{|u(s,a)|}}\right) \leq \\
            &P(f^{+} \geq \frac{k_{m}}{2}) = P(f^{+} -\frac{k_{m}}{5}\geq \frac{3k_{m}}{10}) \leq P(f^{+} -Ef^{+}\geq \frac{3k_{m}}{10}) \leq \\ &\exp^{\frac{-2\left(\frac{3k_{m}}{10}\right)^{2}}{k_{m}}} \leq  \frac{\delta}{8N\ceil{1+\log_{2}\frac{k}{k_{m}}}(SA)^{2}}
    \end{aligned}
\end{equation*}
Where the last inequality stems from our definition of $k_{m}$. By the same argument 
\begin{equation*}
    \begin{aligned}
            &P\left(\mathrm{median}\{G^{\pi^{Q}}(Q,u(s,a),1),...,G^{\pi^{Q}}(Q,u(s,a),k_{m})\}-B^{\pi^{Q}}Q(s,a) \geq \frac{\epsilon_{b}}{\sqrt{|u(s,a)|}}\right) \leq \\
            &P(f^{-} \geq \frac{k_{m}}{2}) = P(f^{-} -\frac{k_{m}}{5}\geq \frac{3k_{m}}{10}) \leq P(f^{-} -Ef^{-}\geq \frac{3k_{m}}{10}) \leq \\ &\exp^{\frac{-2\left(\frac{3k_{m}}{10}\right)^{2}}{k_{m}}} \leq  \frac{\delta}{8N\ceil{1+\log_{2}\frac{k}{k_{m}}}(SA)^{2}}
    \end{aligned}
\end{equation*}
The rest follows from the definition of $F$ as a median with an additive exploration term \eqref{eq:DefineF}. 
\end{proof}
We have seen that the operator $F$ is close to the Bellman operator with some probability for the real $Q$-function the learner possesses. Now we shall show that for each agent, the same holds for the estimated Q-function it holds. This means that we have to show that the above relation holds for $\hat{Q}^{i}_{sat}$ and for all agents. 
We will use the following assumption over the noise terms. 
\begin{restatable}[]{myassump}{ConcentrationBound}
\label{assump:Concentration Bound}
For all agents $i \in \{1,...,N\}$, estimated $Q$-functions $\hat{Q}^{i}$ and state actions $(s,a)$ during the run of the algorithm, the noise terms $n^{i}$  in \eqref{eq: learner and agent noise} satisfy 
\begin{equation} \label{eq:ConcBound}
\begin{aligned}
    P(|n^{i}(s,a)-En^{i}(s,a)|\leq n_{c}^{i}) &\geq 1-\delta_{L}\\ |En^{i}(s,a)| &\leq \mu_{c}^{i} 
\end{aligned}
\end{equation}
for some $\delta_{L},\mu_{c}^{i},n_{c}^{i}$ given $(s,a)$.\footnote{Note that this relation holds for a noise term given the state-action $(s,a)$, and does not include the probability of encountering $(s,a)$ itself}  
\end{restatable} 
For convenience, we define the following events
\begin{equation*}
\begin{aligned}
   I_{1} &\triangleq  \{F^{\pi^{*}}(Q,u(s,a))-B^{\pi^{*}}Q(s,a) \leq 0\} ~,\\
   J_{1}&\triangleq \left\{F^{\pi^{Q}}(Q,u(s,a))-B^{\pi^{Q}}Q(s,a) \geq 2\frac{\epsilon_{b}}{\sqrt{|u(s,a)|}}\right\} ~,\\
    I_{2} &\triangleq 
    \{F^{\pi^{*}}(\hat{Q}^{i}_{sat},u(s,a))-B^{\pi^{*}}\hat{Q}^{i}_{sat}(s,a) \leq -2\gamma(n_{c}^{i}+\mu_{c}^{i})\} ~,\\
   J_{2} &\triangleq \left\{F^{\pi^{\hat{Q}^{i}_{sat}}}(\hat{Q}^{i}_{sat},u(s,a))-B^{\pi^{\hat{Q}^{i}_{sat}}}\hat{Q}^{i}_{sat}(s,a) \geq 2\frac{\epsilon_{b}}{\sqrt{|u(s,a)|}}+2\gamma(n_{c}^{i}+\mu_{c}^{i})\right\}~,\\
   L_{2} &\triangleq \{|\hat{Q}^{i}_{sat}(s,a)-\widetilde{B}\hat{Q}^{i}_{sat}(s,a)| > \epsilon_{a}+(1+\gamma) (n_{c}^{i}+\mu_{c}^{i})\}~.
\end{aligned}
\end{equation*}
\begin{restatable}[]{mylemma}{lemma2}
\label{lemma:lemma2}
Let $\sigma$ be defined as in Definition \ref{def: sigma},$\epsilon_{b}$, and $k_{m}$ as in Lemma \ref{lemma:GtoB}.  Under assumptions \ref{assump: Noise_assumptions1} and \ref{assump:Concentration Bound}, given a fixed $Q$, a fixed state action $(s,a)$ and a fixed agent $i$
\begin{equation*}
\begin{aligned}
   P(I_{2}) &\leq \delta_{L}SA+\frac{\delta}{8N\ceil{1+\log_{2}\frac{k}{k_{m}}}(SA)^{2}} \\
  P(J_{2}) &\leq \delta_{L}SA+\frac{\delta}{8N\ceil{1+\log_{2}\frac{k}{k_{m}}}(SA)^{2}}\\
  P(L_{2}) &\leq \delta_{L}SA
\end{aligned}
\end{equation*}
\end{restatable}
\begin{proof}
Let us denote
\begin{equation*}
    W_{\hat{Q}^{i}} \triangleq \left\{\forall (\overline{s},\overline{a}) :|n^{i}(\overline{s},\overline{a})-En^{i}(\overline{s},\overline{a})|\leq n_{c}^{i} \right\}
\end{equation*}
As the event for which \eqref{eq:ConcBound} holds for all state actions corresponding to $\hat{Q}^{i}$ stemming from the fixed $Q$. Since there are exactly $SA$ noise terms in this event, we have from the Union Bound that $P(W_{\hat{Q}^{i}}) \geq 1- \delta_{L}SA$. We shall show that
\begin{equation*}
\begin{aligned}
   P(I_{2}|W_{\hat{Q}^{i}})&\leq \frac{\delta}{8N\ceil{1+\log_{2}\frac{k}{k_{m}}}(SA)^{2}} \\
   P(J_{2}|W_{\hat{Q}^{i}})&\leq \frac{\delta}{8N\ceil{1+\log_{2}\frac{k}{k_{m}}}(SA)^{2}}\\
   P(L_{2}|W_{\hat{Q}^{i}})&=0 
\end{aligned}
\end{equation*}
We then have
\begin{equation*}
    \begin{aligned}
       P(I_{2}) &=P(I_{2}|W_{\hat{Q}^{i}})P(W_{\hat{Q}^{i}})+P(I_{2}|\overline{W_{\hat{Q}^{i}}})P(\overline{W_{\hat{Q}^{i}}}) \leq  \frac{\delta}{8N\ceil{1+\log_{2}\frac{k}{k_{m}}}(SA)^{2}}\cdot 1+1\cdot \delta_{L}SA  
     \end{aligned}
\end{equation*}
and similarly for $J_{2}$, and
\begin{equation*}
    \begin{aligned}
       P(L_{2})=P(L_{2}|W_{\hat{Q}^{i}})P(W_{\hat{Q}^{i}})+P(L_{2}|\overline{W_{\hat{Q}^{i}}})P(\overline{W_{\hat{Q}^{i}}}) \leq 0\cdot 1+1\cdot \delta_{L}SA = \delta_{L}SA~.
    \end{aligned}
\end{equation*}
Now, to prove the first probability bound above, we assume that the event $W_{\hat{Q}^{i}}$ is true  and that $\overline{I}_{1}$ is also true, where for event $X$, $\overline{X}$ is its complement. We then have
\begin{equation*}
    \begin{gathered}
           \forall (\overline{s},\overline{a}): |n^{i}(\overline{s},\overline{a})-En^{i}(\overline{s},\overline{a})|\leq n_{c}^{i}
   \end{gathered}
\end{equation*}
implying that 
\begin{equation*}
    \begin{aligned}
            En^{i}(\overline{s},\overline{a})-n_{c}^{i}&\leq n^{i}(\overline{s},\overline{a}) \leq n_{c}^{i}+En^{i}(\overline{s},\overline{a}) \\ 
           -\mu_{c}^{i}-n_{c}^{i}&\leq n^{i}(\overline{s},\overline{a}) \leq n_{c}^{i}+\mu_{c}^{i}
   \end{aligned}
\end{equation*}
By using Definition \ref{def:Qsat} for the saturated $Q$-function, we know that for each $(\overline{s},\overline{a})$
\begin{equation*}
    \begin{aligned}
       \hat{Q}^{i}_{sat}(\overline{s},\overline{a}) &=\max\{0,\min\{Q_{\max},\hat{Q}^{i}(\overline{s},\overline{a})\}\} \\
       &\geq  \max\{0,\min\{Q_{\max},Q(\overline{s},\overline{a})-n_{c}^{i}-\mu_{c}^{i}\}\} \\
       &\geq Q(\overline{s},\overline{a}) -n_{c}^{i}-\mu_{c}^{i}
    \end{aligned}
\end{equation*}
By using the definition of the Bellman operator
\begin{equation*}
    \begin{aligned}
       B^{\pi^{*}}\hat{Q}^{i}_{sat}(s,a) &= \sum_{s'}p(s'|s,a)\left[R(s,\pi^{*}(s),s')+\gamma \hat{Q}^{i}_{sat}(s',\pi^{*}(s'))\right] \\ 
       &\leq  \sum_{s'}p(s'|s,a)\left[R(s,\pi^{*}(s),s')+\gamma Q(s',\pi^{*}(s')) + \gamma n_{c}^{i}+\gamma \mu_{c}^{i}\right] \\ 
       &= B^{\pi^{*}}Q(s,a) + \gamma n_{c}^{i}+\gamma \mu_{c}^{i}
    \end{aligned}
\end{equation*}
where we used the assumption that the noise is bounded for \textbf{all state-actions}. By using the same trick for the operator $G^{\pi^{*}}$ we have that for all $j$
\begin{equation*}
    \begin{aligned}
       G^{\pi^{*}}(\hat{Q}^{i}_{sat},u(s,a),j) 
       &= \frac{k_{m}}{|u(s,a)|}\sum_{l}(\widetilde{R}_{n(l),t(l)}+\gamma \hat{Q}^{i}_{sat}(s_{l}^{'},\pi^{*}(s_{l}^{'})))\\ 
       &\geq
        \frac{k_{m}}{|u(s,a)|}\sum_{l}(\widetilde{R}_{n(l),t(l)}+\gamma Q(s_{l}^{'},\pi^{*}(s_{l}^{'}))-\gamma n_{c}^{i}-\gamma \mu_{c}^{i}) \\ &=G^{\pi^{*}}(Q,u(s,a),j) -\gamma n_{c}^{i}-\gamma \mu_{c}^{i} 
    \end{aligned}
\end{equation*}
We then have that
\begin{equation*}
    \begin{aligned}
       F^{\pi^{*}}(\hat{Q}^{i}_{sat},u(s,a)) &= \frac{\epsilon_{b}}{\sqrt{|u(s,a)|}}+\overline{\mathrm{med}} \\ 
       &\geq \frac{\epsilon_{b}}{\sqrt{|u(s,a)|}}+\overline{\mathrm{med}} - \gamma n_{c}^{i}-\gamma \mu_{c}^{i}\\ 
       &=F^{\pi^{*}}(Q,u(s,a)) - \gamma n_{c}^{i}-\gamma \mu_{c}^{i}
    \end{aligned}
\end{equation*}
where $\overline{\mathrm{med}}=\mathrm{median}\{G^{\pi^{*}}(\hat{Q}^{i}_{sat},u(s,a),1),...,G^{\pi^{*}}(\hat{Q}^{i}_{sat},u(s,a),k_{m})\}$. 
Therefore 
\begin{equation*}
\begin{aligned}
          F^{\pi^{*}}(\hat{Q}^{i}_{sat},u(s,a))-B^{\pi^{*}}\hat{Q}^{i}_{sat}(s,a) &\geq F^{\pi^{*}}(Q,u(s,a))-B^{\pi^{*}}Q(s,a)\\ 
          &\quad -2\gamma(n_{c}^{i}+\mu_{c}^{i}) \\
          &\geq -2\gamma(n_{c}^{i}+\mu_{c}^{i}) 
\end{aligned}
\end{equation*}
Since $\overline{I}_{1} \subseteq \overline{I}_{2}$ given $W_{\hat{Q}^{i}}$
\begin{equation*}
  P(I_{2}|W_{\hat{Q}^{i}}) \leq P(I_{1}|W_{\hat{Q}^{i}}) \leq \frac{\delta}{8N\ceil{1+\log_{2}\frac{k}{k_{m}}}(SA)^{2}}
\end{equation*} 
Note that from assumption \ref{assump: Noise_assumptions1}, the noise terms $n^{i}$ are independent of $Q$ and of the samples collected so far, which means that the events $I_{1}$ and $W$ are independent so that $P(I_{1}|W)=P(I_{1})$. We use the same principle to prove the second inequality. Note that in that case, since we use the Bellman operator and the $F$ operator with the greedy policy, we use the fact that $\max_{a'}\hat{Q}^{i}_{sat}(s',a') \geq \max_{a'}Q(s',a') - n_{c}^{i}-\mu_{c}^{i}$. For the third inequality, note that we know that the Bellman iterations for the $Q$-function $Q$ have converged with probability 1, 
\begin{equation*}
    \begin{gathered}
         |Q(s,a)-\widetilde{B}Q(s,a)| \leq \epsilon_{a} 
    \end{gathered}
\end{equation*}
and using the definition of $\widetilde{B}$
\begin{equation*}
    \begin{aligned}
         \widetilde{B}\hat{Q}^{i}_{sat}(s,a)&=\min\{Q_{\max},\max\{0,\widetilde{F}(\hat{Q}^{i}_{sat},u(s,a))\}\} \\
         &\leq \min\{Q_{\max},\max\{0,\widetilde{F}(Q,u(s,a))+\gamma(n_{c}^{i}+\mu_{c}^{i})\}\} \\ 
         &\leq \widetilde{B}Q(s,a) + \gamma(n_{c}^{i}+\mu_{c}^{i})
    \end{aligned}
\end{equation*}
and
\begin{equation*}
    \begin{aligned}
         \widetilde{B}\hat{Q}^{i}_{sat}(s,a)&=\min\{Q_{\max},\max\{0,\widetilde{F}(\hat{Q}^{i}_{sat},u(s,a))\}\} \\
         &\geq \min\{Q_{\max},\max\{0,\widetilde{F}(Q,u(s,a))-\gamma(n_{c}^{i}+\mu_{c}^{i})\}\} \\ 
         &\geq \widetilde{B}Q(s,a) - \gamma(n_{c}^{i}+\mu_{c}^{i})
    \end{aligned}
\end{equation*}
Therefore, in that case
\begin{equation*}
    \begin{aligned}
         \hat{Q}^{i}_{sat}(s,a)-\widetilde{B}\hat{Q}^{i}_{sat}(s,a) &\leq Q(s,a)-\widetilde{B}Q(s,a) +(1+\gamma)(n_{c}^{i}+\mu_{c}^{i}) \\ &\leq \epsilon_{a}+(1+\gamma)(n_{c}^{i}+\mu_{c}^{i})
    \end{aligned}
\end{equation*}
And by the same arguments
\begin{equation*}
    \begin{aligned}
         \hat{Q}^{i}_{sat}(s,a)-\widetilde{B}\hat{Q}^{i}_{sat}(s,a) &\geq Q(s,a)-\widetilde{B}Q(s,a) -(1+\gamma)(n_{c}^{i}+\mu_{c}^{i}) \\
         &\geq -\epsilon_{a}-(1+\gamma)(n_{c}^{i}+\mu_{c}^{i})
    \end{aligned}
\end{equation*}
Therefore, $P(L_{2}|W_{\hat{Q}^{i}})=0$ as desired. 
\end{proof} 


We have proven a probability bound for a given $Q$-function, state-action $(s,a)$ and an agent (which received the above mentioned version of the $Q$-function). Next, we use the union bound to show that this holds for all $Q$-functions and agents during the run of the algorithm. 

\begin{restatable}[]{mylemma}{lemma3}
\label{lemma:lemma3}
Let $\sigma$ be defined as in Definition \ref{def: sigma}, $\epsilon_{b}^{2}=4k_{m}(\sigma^{2}+\sigma^{2}_{R})$, and $k_{m}=\ceil{5.6\ln\frac{8N\ceil{1+\log_{2}\frac{k}{k_{m}}}(SA)^{2}}{\delta}}$.  Assume that assumptions \ref{assump: Noise_assumptions1} and \ref{assump:Concentration Bound} hold for all $Q$-functions, state-actions and agents that can be encountered during the run of the algorithm. Then the events $\overline{I}_{2},\overline{J}_{2},\overline{L}_{2}$ defined before Lemma \ref{lemma:lemma2} occur for all the noisy versions that can be encountered during the run of Algorithm \ref{alg:Communication noise tolerant PAC Exploration} $\hat{Q}^{i}$, for all state-actions and agents simultaneously with probability larger than $1-3N\ceil{1+\log_{2}\frac{k}{k_{m}}}(SA)^{3}\delta_{L}-\frac{\delta}{4}$. 
\end{restatable}
\begin{proof}
We use a union bound over the number of agents, the number of different $Q$-functions the learner produces that can be encountered during the run of the algorithm, and the number of state-actions. Note that a particular $Q$-function $Q$ is determined by the samples that have been used to calculate it (via Value Iteration) up to that point. As seen in Algorithm \ref{alg:Communication noise tolerant PAC Exploration}, there are at most $\ceil{1+\log_{2}\frac{k}{k_{m}}}$ distinct sample sets with $|u(s,a)|>0$ for a given state-action, meaning that there are at most $\ceil{1+\log_{2}\frac{k}{k_{m}}}SA$ distinct $Q$-functions encountered in practice. 

Therefore, there are at most $N(SA)^{2}\ceil{1+\log_{2}\frac{k}{k_{m}}}$ possibilities to include in the union bound sum, for each one of the three events above. This lemma is then proved by using the union bound over the results of Lemma \ref{lemma:lemma2}.
\end{proof} 
In lemma \ref{lemma:lemma4}, we show that the events $\overline{I}_{2},\overline{J}_{2}$ bounding the difference between $F^{\pi^{*}}$ to $B^{\pi^{*}}$ and $F^{\pi^{\hat{Q}^{i}_{sat}}}$ to $B^{\pi^{\hat{Q}^{i}_{sat}}}$ hold not only for the function $F$, but also for the estimated $Q$ functions $\hat{Q}^{i}_{sat}$ used by the different agents. 
\begin{restatable}[]{mylemma}{lemma4}
\label{lemma:lemma4}
Let $\sigma$ be defined as in Definition \ref{def: sigma}, $\epsilon_{b}^{2}=4k_{m}(\sigma^{2}+\sigma^{2}_{R})$, and $k_{m}=\ceil{5.6\ln\frac{8N\ceil{1+\log_{2}\frac{k}{k_{m}}}(SA)^{2}}{\delta}}$.  Assume that assumptions \ref{assump: Noise_assumptions1} and \ref{assump:Concentration Bound} hold for all $Q$-functions, state-actions and agents that can be encountered during the run of the algorithm. Then
\begin{equation*}
    \begin{aligned}
       \hat{Q}^{i}_{sat}(s,a)-B^{\pi^{*}}\hat{Q}^{i}_{sat}(s,a) &> -\epsilon_{a}-(n_{c}^{i}+\mu_{c}^{i})(1+3\gamma) \\
  \hat{Q}^{i}_{sat}(s,a)-B^{\pi^{Q}}\hat{Q}^{i}_{sat}(s,a) &< \epsilon_{a}+\frac{2\epsilon_{b}}{\sqrt{|u(s,a)|}}+(n_{c}^{i}+\mu_{c}^{i})(1+3\gamma)
  \end{aligned}
\end{equation*}
for all $Q$-functions, state-actions $(s,a)$ and for all agents simultaneously with probability larger than $1-3N\ceil{1+\log_{2}\frac{k}{k_{m}}}(SA)^{3}\delta_{L}-\frac{\delta}{4}$~.
\end{restatable}
\begin{proof}
Until collected samples are used for value iteration, we have that for all agents and state-actions $\hat{Q}^{i}_{sat}=Q_{\max}$. Since $B^{\pi^{*}}\hat{Q}^{i}_{sat}(s,a) \leq Q_{\max}$. we have that $\hat{Q}^{i}_{sat}(s,a)-B^{\pi^{*}}\hat{Q}^{i}_{sat}(s,a) \geq 0 > -\epsilon_{a}-(n_{c}^{i}+\mu_{c}^{i})(1+3\gamma)$ as wanted (as we shall see in Theorem \ref{theo:sample complexity}, we won't be needing the second inequality in this lemma for the case of such a $Q$-function).

Otherwise, we have that under the probability mentioned above, for all $Q$-functions, state-actions and agents that can be encountered during the run of the algorithm $(s,a,\hat{Q}^{i})$ (see Lemma \ref{lemma:lemma3}),
\begin{equation*}
    \begin{aligned}
		B^{\pi^{*}}\hat{Q}^{i}_{sat}(s,a) &= \min\{Q_{\max},\max\{0,B^{\pi^{*}}\hat{Q}^{i}_{sat}(s,a)\}\} \\
		&\leq \min\left\{Q_{\max},\max\{0,F^{\pi^{*}}(\hat{Q}^{i}_{sat},u(s,a)))+2\gamma(n_{c}^{i}+\mu_{c}^{i})\}\right\} \\
 		&\leq \widetilde{B}^{\pi^{*}}\hat{Q}^{i}_{sat}(s,a)+2\gamma(n_{c}^{i}+\mu_{c}^{i}) \\  
		&\leq \widetilde{B}\hat{Q}^{i}_{sat}(s,a)+2\gamma(n_{c}^{i}+\mu_{c}^{i})  \\
		&\leq \hat{Q}^{i}_{sat}(s,a)+\epsilon_{a}+(1+3\gamma)(n_{c}^{i}+\mu_{c}^{i})
    \end{aligned}
\end{equation*}
The first equality stems from the fact that we use the true Bellman operator on a $Q$-function that is bounded in $[0,Q_{max}]$, and the result of applying the operator to this function is also bounded. The inequality that follows stems from $\overline{I}_{2}$. The second and third inequalities are due to the definition of the approximate Bellman operator $\widetilde{B}$ (Definition \ref{def: B operator}),  and the last inequality stems from $\overline{L}_{2}$. Similarly
\begin{equation*}
    \begin{aligned}
		B^{\pi^{\hat{Q}^{i}_{sat}}}\hat{Q}^{i}_{sat}(s,a) &= \min\{Q_{\max},\max\{0,B^{\pi^{\hat{Q}^{i}_{sat}}}\hat{Q}^{i}_{sat}(s,a)\}\} \\
		&\geq \min\{Q_{\max},\max\{0,F^{\pi^{\hat{Q}^{i}_{sat}}}(\hat{Q}^{i}_{sat},u(s,a)))-2\gamma(n_{c}^{i}+\mu_{c}^{i})-\widetilde\epsilon_b\}\} \\ 
		&\geq \widetilde{B}^{\pi^{\hat{Q}^{i}_{sat}}}\hat{Q}^{i}_{sat}(s,a)-2\gamma(n_{c}^{i}+\mu_{c}^{i})-\widetilde\epsilon_b \\
		&= \widetilde{B}\hat{Q}^{i}_{sat}(s,a)-2\gamma(n_{c}^{i}+\mu_{c}^{i})-\widetilde\epsilon_b \\
		&\geq \hat{Q}^{i}_{sat}(s,a)-\epsilon_{a}-(1+3\gamma)(n_{c}^{i}+\mu_{c}^{i}) -\widetilde\epsilon_b
    \end{aligned}
\end{equation*}
where $\widetilde\epsilon_b=2\epsilon_{b}/\sqrt{|u(s,a)|}$.
\end{proof}
Before we continue, we quote two lemmas from \citep{Pazis2016a} that will be useful ahead. Lemma \ref{lemma:Bern} is a basic result regarding Bernoulli random variables. 
\begin{restatable}[]{mylemma}{Bern}
\label{lemma:Bern}
Let $Y_{i}$ for $i \in \{1,...,n\}$ be independent Bernoulli random variables such that $P(Y_{i}=1) \geq p_{i} \quad \forall i$. for some $m,\delta > 0$ , if $\frac{2}{m}\ln\frac{1}{\delta}<1$ and
\begin{equation*}
    \sum_{i=1}^{n} \geq \frac{m}{1-\sqrt{\frac{2}{m}\ln\frac{1}{\delta}}}
\end{equation*}
then 
\begin{equation*}
    P\left(\sum_{i=1}^{n}Y_{i} \geq m \right) \geq 1- \delta
\end{equation*}
\end{restatable}
Lemma \ref{lemma:BellmanDiff} is a basic result regarding RL in MDPs, and bounds the difference between the optimal value function and a value function resulting from a given greedy policy, given some MDP in which the Bellman error is bounded.
\begin{restatable}[]{mylemma}{BellmanDiff}
\label{lemma:BellmanDiff}
Given some $Q$-function $Q$ and non-negative constants $\epsilon_{*},\epsilon_{\pi^{Q}},\epsilon_{1},...,\epsilon_{n}$ such that $\epsilon_{\pi^{Q}}\leq \epsilon_{i} \quad \forall i$, let $Q(s,a)-B^{\pi^{*}}Q(s,a) \geq \epsilon_{*}$ for all $(s,a)$, and let $X_{1},...,X_{n}$ be sets of state-actions such that $Q(s,a)-B^{\pi^{Q}}Q(s,a) \leq \epsilon_{i}$ for all states-actions $(s,a)$ in $X_{i}$, where $\pi^{Q}$ is the greedy policy over $Q$. Also, let $Q(s,a)-B^{\pi^{Q}}Q(s,a) \leq \epsilon_{\pi^{Q}}$ for all $(s,a) \notin \cup_{i=1}^{n}X_{i}$. Let $T_{H}=\ceil{\frac{1}{1-\gamma}\ln\frac{Q_{\max}}{\epsilon_{s}}}$ and define $H=\{1,2,4,...,2^{i}\}$ where i is the largest integer such that $2^{i} \leq T_{H}$. Define $p_{h,i}(s)$ for an integer $h \geq 0$ to be the probability of encountering exactly $h$ state-actions $(s,a)$ for which $(s,a) \in X_{i}$ when starting from state $s$ and following $\pi^{Q}$ for a totall of $\min\{T,T_{H}\}$ for some $T$. Finally, let $p^{e}_{h,i}(s)=\sum_{m=h}^{2h-1}p_{m,i}(s)$. Then 
\begin{equation*}
    V^{*}(s)-V^{\pi^{Q}}(s) \leq \frac{\epsilon_{*}+\epsilon_{\pi^{Q}}}{1-\gamma} +\epsilon_{s}+\epsilon_{e}
\end{equation*}
Where $\epsilon_{e} \triangleq 2\sum_{i=1}^{n}\sum_{h\in H}hp^{e}_{h,i}(s)(\epsilon_{i}-\epsilon_{\pi^{Q}})+\gamma^{T}Q_{max}$.
\end{restatable}
We will use lemma \ref{lemma:BellmanDiff} together with \ref{lemma:lemma5} and \ref{lemma:lemma4} to prove the main theorem \ref{theo:sample complexity}. Lemma \ref{lemma:lemma5} bounds the term $\epsilon_{e}$ in lemma \ref{lemma:BellmanDiff} for our algorithm. 
\begin{restatable}[]{mylemma}{lemma5}
\label{lemma:lemma5}
Let $(s_{1,i},s_{2,i},...,)$ for $i \in \{1,...,N\}$ be the random paths generated on some execution of Algorithm \ref{alg:Communication noise tolerant PAC Exploration}. Let $\tau(t)$ be the number of steps from step t to the next step step $t'$ for which the policy changes. Let $T_{H}=\ceil{\frac{1}{1-\gamma}\ln\frac{Q_{\max}}{\epsilon_{s}}}$ and define $H=\{1,2,4,...,2^{i}\}$ where i is the largest integer such that $2^{i} \leq T_{H}$. Let $K_{a}=\{k_{m},2k_{m},2^{2}k_{m},...,k\}$. Let $k_{a}^{-}$ be the largest value in $K_{a}$ that is strictly smaller than $k_{a}$, or 0 if such a value does not exist. Let $X_{k_{a}}(t)$ be the set of state-actions $(s,a)$ at step t for which $k_{a}^{-}=|u(s,a)|$. Define $p_{h,k_{a}}(s_{t,i})$ for $k_{a} \in K_{a}$ to be the following conditional probability: Given $\hat{Q}^{i}_{sat}$ at step t, exactly $h$ state-actions in $X_{k_{a}}(t)$ are encountered by agent $i$ during the next $\min\{T_{H},\tau(t)\}$ steps. Let $p_{h,k_{a}}^{e}(s_{t,i}) \triangleq \sum_{m=h}^{2h-1}p_{m,k_{a}}(s_{t,i})$. If $\frac{2T_{H}}{NSA}\ln\frac{2|K_{a}|}{\delta}<1$ and $SA\geq 2$, with probability $1-\delta/2$
\begin{equation*}
    \begin{gathered}
       \sum_{i=1}^{N}\sum_{t=0}^{\infty}\sum_{h \in H}hp_{h,k_{a}}^{e}(s_{t,i}) < \frac{SA(k_{a}+N)(1+\log_{2}T_{H})T_{H}}{1-\sqrt{\frac{2T_{H}}{SA(k_{a}+N)}\ln\frac{2|K_{a}|}{\delta}}}
    \end{gathered}
\end{equation*}
for all $k_{a} \in K_{a}$ simultaneously. 
\end{restatable}
\begin{proof}
Let us fix some $k_{a}$ and denote by $Y^{e}_{h,k_{a}}(s_{t,i})$ the Bernoulli random variables corresponding to the probabilities $p_{h,k_{a}}^{e}(s_{t,i})$. Since we condition these variables and are only interested in their outcome given this condition, we have from the Markov property that variables $Y^{e}_{h,k_{a}}(s_{t,i})$ at least $T_{H}$ time steps apart are independent, and variables for different agents are independent as well. Define $T^{H}_{j}$ for $j \in \{0,1,...,T_{H}-1\}$ to be the infinite set of time steps such $T^{H}_{j} = \{j,j+T_{H},j+2T_{H},...\}$. From Algorithm \ref{alg:Communication noise tolerant PAC Exploration}, $k_{a}$ samples will be stored in the temporary sample set $u^{tmp}(s,a)$ before a state action $(s,a)$ with $|u(s,a)|=k_{a}^{-}$ advances to have $k_{a}$ samples. And since $N$ agents are exploring in parallel, at most $k_{a}+N$ samples will be collected for such a state-action before it progresses to have $|u(s,a)|=k_{a}$ (this is the worst case where all the agents visit the same $(s,a)$ at the last time step before its sample set is updated). 

Let us assume that there exists a $j \in \{0,1,...,T_{H}-1\}$ and $h \in H$ such that 
\begin{equation*}
    \sum_{i=1}^{N}\sum_{t\in T^{H}_{j}}p_{h,k_{a}}^{e}(s_{t,i}) \geq \frac{SA(k_{a}+N)}{h\left(1-\sqrt{\frac{2h}{SA(k_{a}+N)}\ln\frac{2|K_{a}|}{\delta}}\right)}
\end{equation*}
However, since this is a  sum of probabilities of independent Bernoulli variables, we have from lemma \ref{lemma:Bern} that 
\begin{equation*}
    P\left(\sum_{i=1}^{N}\sum_{t\in T^{H}_{j}}Y_{h,k_{a}}^{e}(s_{t,i}) \geq \frac{SA(k_{a}+N)}{h}\right) \geq 1-\frac{\delta}{2|K_{a}|}
\end{equation*}
This event is a contradiction, since it means that more than $SA(k_{a}+N)$ samples are collected for state-actions with $k_{a}^{-}$ samples. Therefore, we have that 
\begin{equation*}
    \sum_{i=1}^{N}\sum_{t\in T^{H}_{j}}p_{h,k_{a}}^{e}(s_{t,i}) < \frac{SA(k_{a}+N)}{h\left(1-\sqrt{\frac{2h}{SA(k_{a}+N)}\ln\frac{2|K_{a}|}{\delta}}\right)}
\end{equation*}
For all  $j \in \{0,1,...,T_{H}-1\}$ and $h \in H$ simultaneously, with a probability larger than $1-\frac{\delta}{2|K_{a}|}$. By summing over the values of $h$ and $j$ we have
\begin{equation*}
    \sum_{i=1}^{N}\sum_{t=0}^{\infty}\sum_{h \in H}hp_{h,k_{a}}^{e}(s_{t,i}) < \frac{SA(k_{a}+N)(1+\log_{2}T_{H})T_{H}}{1-\sqrt{\frac{2T_{H}}{SA(k_{a}+N)}\ln\frac{2|K_{a}|}{\delta}}}
\end{equation*}
with a probability larger than $1-\frac{\delta}{2|K_{a}|}$. Since this is true for a fixed value of $k_{a}$, we now use the Union Bound to conclude that the previous equation holds for all $k_{a} \in K_{a}$ simultaneously with probability larger than $1-\frac{\delta}{2}$. 

Note that Since we have defined the settings of Algorithm \ref{alg:Communication noise tolerant PAC Exploration} so that a learner only sends a new $Q$-function to the agents when its own $Q$-function is updated, the agents have a noisy but constant version of the $Q$-function $\hat{Q}^{i}_{sat}$ for $\tau(t)$ steps - which is crucial for the proof. It also agrees with the determination principle for multi-agent systems defined in \citep{Dimakopoulou2018}.
\end{proof}
We are now ready to present the main sample complexity bound. As stated before, this theorem combines the lemmas \ref{lemma:lemma4},\ref{lemma:lemma5},\ref{lemma:BellmanDiff} to prove PAC results. 
\begin{restatable}[]{mytheo}{samplecomplexity}
\label{theo:sample complexity}
Under assumptions \ref{assump:main}, \ref{assump: Noise_assumptions1} and \ref{assump:Concentration Bound}, assume further that $\delta_{L}\leq\frac{\delta}{12N\ceil{1+\log_{2}\frac{k}{k_{m}}}(SA)^{3}}$. Then with probability at least $1-\delta$, for all $t$ and $i$, 
\begin{equation*}
    \begin{gathered}
       V^{\widetilde{\pi}_{i}}(s_{t,i}) \geq V^{*}(s_{t,i}) - \frac{2\epsilon_{a}+2(1+3\gamma)(n_{c}^{i}+\mu_{c}^{i})}{1-\gamma} -3\epsilon_{s} -\epsilon_{e}^{i}(t)
    \end{gathered}
\end{equation*}
where
\begin{equation*}
    \begin{gathered}
       TCE = \sum_{i=1}^{N}\sum_{t=0}^{\infty}\epsilon_{e}^{i}(t) = \widetilde{O}\left(\left(N\left(Q_{\max}+\sqrt{\sigma^{2}+\sigma^{2}_{R}}\right)+\frac{\sigma^{2}+\sigma^{2}_{R}}{\epsilon_{s}(1-\gamma)}\right)SA/(1-\gamma)\right)~,  
    \end{gathered}
\end{equation*}
and $\widetilde{O}$ stands for a big-O up to logarithmic terms.
\end{restatable}
\begin{proof}
In short, the non-stationary policy of each one of the agents can be broken up into fixed-policy segments, in which we follow an estimated $Q$-function greedily. Given that the Bellman error for each segment is acceptably small, lemma \ref{lemma:BellmanDiff} states that the greedy policy for that segment has a bounded error with respect to the optimal policy . We use Lemma \ref{lemma:lemma4} to show that with high probability, the Bellman errors of all $Q$-functions for all agents during the run of the algorithm are bounded, and then combine it with Lemma \ref{lemma:lemma5} which bounds the number of times we can encounter state-actions for which we haven't collected enough samples yet. Combining these three theorems, we get to a sample complexity bound for the algorithm as a whole. \\\\
Corresponding to lemma \ref{lemma:BellmanDiff}, the groups we divide the state-actions to for each agent are $X_{k_{a}}$ as defined in Lemma \ref{lemma:lemma5}. We now have to show that state-actions in each group have the same bound on their Bellman error (while different groups may have different bounds). According to Lemma \ref{lemma:lemma4} we have that with probability larger than $1-\frac{\delta}{2}$, for all state-actions during the run of the algorithm
\begin{equation*}
    \begin{gathered}
              \hat{Q}^{i}_{sat}(s,a)-B^{\pi^{*}}\hat{Q}^{i}_{sat}(s,a) > -\epsilon_{a}-(n_{c}^{i}+\mu_{c}^{i})(1+3\gamma)
    \end{gathered}
\end{equation*}
For the required upper bound on the Bellman error, we will divide the state-actions to groups.
\begin{enumerate}
    \item For the first group $X_{k_{m}}$ note that none of the previous theorems holds, since these are state-actions with no samples in their sample set. Therefore, we simply use the fact that the $Q$-functions are all bounded by our definition:
    \begin{equation*}
        \forall (s,a) \in X_{k_{m}}: \hat{Q}^{i}_{sat}(s,a)-B^{\pi^{\hat{Q}^{i}_{sat}}}\hat{Q}^{i}_{sat}(s,a) \leq Q_{\max}
    \end{equation*}
    \item For all $(s,a) \in X_{k_{a}},k_{a}>0$, we have again from Lemma \ref{lemma:lemma4} that
    \begin{equation*}
        \begin{aligned}
                  \hat{Q}^{i}_{sat}(s,a)-B^{\pi^{Q}}\hat{Q}^{i}_{sat}(s,a) < \epsilon_{a}+(n_{c}^{i}+\mu_{c}^{i})(1+3\gamma)+2\sqrt{\frac{8k_{m}(\sigma^{2}+\sigma^{2}_{R})}{k_{a}}}
        \end{aligned}
    \end{equation*}
    Where we substitute $\epsilon_{b}^{2}=4k_{m}(\sigma^{2}+\sigma_{R}^{2})$, and use the fact that the fact that $k_{a}^{-}\geq\frac{k_{a}}{2}$.
    \item Finally, for state-actions with $k$ samples in the sample set, we can use the assumption $k \geq \ceil{\epsilon_{b}^{2}/((1-\gamma)^2\epsilon_{s}^2)}$ and have
    \begin{equation*}
        \begin{gathered}
                  \hat{Q}^{i}_{sat}(s,a)-B^{\pi^{Q}}\hat{Q}^{i}_{sat}(s,a) < \epsilon_{a}+(n_{c}^{i}+\mu_{c}^{i})(1+3\gamma)+2(1-\gamma)\epsilon_{s}
        \end{gathered}
    \end{equation*}
\end{enumerate}
Now, we know that even though the policy $\widetilde{\pi}_{i}$ is non-stationary, it is comprised of stationary segments. Using the definitions from Lemma \ref{lemma:lemma5}, starting from step t for agent $i$, $\widetilde{\pi}_{i}$ is equal to $\pi^{\hat{Q}^{i}_{sat}}$ for at least $\tau(t)$ steps. Combining the results from Lemma \ref{lemma:lemma5} with the bounds we have listed above, we have that with a probability larger than $1-\delta$, for all time steps and all agents:
\begin{equation*}
    \begin{gathered}
       V^{\widetilde{\pi}_{i}}(s_{t,i}) \geq V^{*}(s_{t,i}) - \frac{2\epsilon_{a}+2(1+3\gamma)(n_{c}^{i}+\mu_{c}^{i})}{1-\gamma} -3\epsilon_{s} -\epsilon_{e}^{i}(t)
    \end{gathered}
\end{equation*}
with
\begin{equation*}
    \begin{aligned}
        \epsilon_{e}^{i}(t) = &\gamma^{\tau(t)}Q_{\max}+2\sum_{h \in H}\left(hp^{e}_{h,k_{m}}(s_{t,i})\right)Q_{\max}+\\
        &\sum_{k_{a} \in \{K_{a}-k_{m}\}}2\sum_{h \in H}\left(hp^{e}_{h,k_{a}}(s_{t,i})\right)2\sqrt{\frac{8k_{m}(\sigma^{2}+\sigma^{2}_{R})}{k_{a}}}
    \end{aligned}
\end{equation*}
We now calculate a bound over the $TCE$. With a probability larger than $1-\delta$ we have
\begin{equation*}
    \begin{aligned}
            \sum_{i=1}^{N}\sum_{t=0}^{\infty}\epsilon_{e}^{i}(t) = &\sum_{i=1}^{N}\sum_{t=0}^{\infty}\left(\gamma^{\tau(t)}Q_{\max}+2\sum_{h \in H}\left(hp^{e}_{h,k_{m}}(s_{t,i})\right)Q_{\max}\right) \\ &+\sum_{i=1}^{N}\sum_{t=0}^{\infty}\sum_{k_{a} \in \{K_{a}-k_{m}\}}2\sum_{h \in H}\left(hp^{e}_{h,k_{a}}(s_{t,i})\right)2\sqrt{\frac{8k_{m}(\sigma^{2}+\sigma^{2}_{R})}{k_{a}}} \\
            &\leq \frac{NSAQ_{max}\ceil{1+\log_{2}\frac{k}{k_{m}}}}{1-\gamma}+\frac{2Q_{max}SA(k_{m}+N)(1+\log_{2}T_{H})T_{H}}{1-\sqrt{\frac{2T_{H}}{SA(k_{m}+N)}\ln\frac{2|K_{a}|}{\delta}}}\\
            &+2\sum_{k_{a} \in \{K_{a}-k_{m}\}}\frac{2\sqrt{\frac{8k_{m}(\sigma^{2}+\sigma^{2}_{R})}{k_{a}}}SA(k_{a}+N)(1+\log_{2}T_{H})T_{H}}{1-\sqrt{\frac{2T_{H}}{SA(k_{a}+N)}\ln\frac{2|K_{a}|}{\delta}}}
    \end{aligned}
\end{equation*}
Here we used the fact there are at most $\ceil{1+\log_{2}\frac{k}{k_{m}}}SA$ policy changes for the first term, and substituted the result from lemma \ref{lemma:lemma5} for the second and third term. Continuing 
\begin{equation*}
    \begin{aligned}
            \sum_{i=1}^{N}\sum_{t=0}^{\infty}\epsilon_{e}^{i}(t) <
            &\frac{Q_{max}SA(2k_{m}+N\ceil{3+\log_{2}\frac{k}{k_{m}}})(1+\log_{2}T_{H})T_{H}}{1-\sqrt{\frac{2T_{H}}{SA(k_{m}+N)}\ln\frac{2|K_{a}|}{\delta}}}\\
            &+\frac{2\sum_{k_{a} \in \{K_{a}-k_{m}\}}2\sqrt{\frac{8k_{m}(\sigma^{2}+\sigma^{2}_{R})}{k_{a}}}SA(k_{a}+N)(1+\log_{2}T_{H})T_{H}}{1-\sqrt{\frac{2T_{H}}{SA(k_{m}+N)}\ln\frac{2|K_{a}|}{\delta}}} \\
            &=\frac{Q_{max}SA(2k_{m}+N\ceil{3+\log_{2}\frac{k}{k_{m}}})(1+\log_{2}T_{H})T_{H}}{1-\sqrt{\frac{2T_{H}}{SA(k_{m}+N)}\ln\frac{2|K_{a}|}{\delta}}}\\
            &+\frac{2\sum_{k_{a} \in \{K_{a}-k_{m}\}}2\sqrt{8k_{m}(\sigma^{2}+\sigma^{2}_{R})}SA(\sqrt{k_{a}}+\frac{N}{\sqrt{k_{a}}})(1+\log_{2}T_{H})T_{H}}{1-\sqrt{\frac{2T_{H}}{SA(k_{m}+N)}\ln\frac{2|K_{a}|}{\delta}}} \\
            &\leq \frac{Q_{max}SA(2k_{m}+N\ceil{3+\log_{2}\frac{k}{k_{m}}})(1+\log_{2}T_{H})T_{H}}{1-\sqrt{\frac{2T_{H}}{SA(k_{m}+N)}\ln\frac{2|K_{a}|}{\delta}}}\\
            &+\frac{\sqrt{8k_{m}(\sigma^{2}+\sigma^{2}_{R})}SA(1+\log_{2}T_{H})T_{H}}{1-\sqrt{\frac{2T_{H}}{SA(k_{m}+N)}\ln\frac{2|K_{a}|}{\delta}}}\\
            &\times\left(4\sqrt{k}\left(1+\sum_{l=0}^{\infty}\left(\frac{1}{2^{l}}+\frac{1}{2^{l}\sqrt{2}}\right)\right)+4N\left(\sum_{l=1}^{\infty}\left(\frac{1}{2^{l}}+\frac{1}{2^{l}\sqrt{2}}\right)\right)\right)
    \end{aligned}
\end{equation*}
Here we used the fact that the number of possible values for $k_{a}$ is bounded. By bounding the sums in the last equation we get to a simpler expression 
\begin{equation*}
    \begin{aligned}
            \sum_{i=1}^{N}\sum_{t=0}^{\infty}\epsilon_{e}^{i}(t) <
            &\frac{Q_{max}SA(2k_{m}+N\ceil{3+\log_{2}\frac{k}{k_{m}}})(1+\log_{2}T_{H})T_{H}}{1-\sqrt{\frac{2T_{H}}{SA(k_{m}+N)}\ln\frac{2|K_{a}|}{\delta}}}\\
            &+\frac{\sqrt{8k_{m}(\sigma^{2}+\sigma^{2}_{R})}SA(1+\log_{2}T_{H})T_{H}}{1-\sqrt{\frac{2T_{H}}{SA(k_{m}+N)}\ln\frac{2|K_{a}|}{\delta}}}(18\sqrt{k}+10N) \\
            &=\frac{SA(1+\log_{2}T_{H})T_{H}}{1-\sqrt{\frac{2T_{H}}{SA(k_{m}+N)}\ln\frac{2|K_{a}|}{\delta}}}\\
            &\times\left(Q_{max}(2k_{m}+N\ceil{3+\log_{2}\frac{k}{k_{m}}})+(18\sqrt{k}+10N)\sqrt{8k_{m}(\sigma^{2}+\sigma^{2}_{R})}\right)
    \end{aligned}
\end{equation*}
We now assume that for a state-action space of realistic size, the denominator is approximately 1 and write
\begin{equation*}
    \begin{aligned}
            \sum_{i=1}^{N}\sum_{t=0}^{\infty}\epsilon_{e}^{i}(t) 
            &< \left(SA(1+\log_{2}T_{H})\ceil{\frac{1}{1-\gamma}\ln\frac{Q_{max}}{\epsilon_{s}}}\right)\\
            &\times \left(Q_{max}(2k_{m}+N\ceil{3+\log_{2}\frac{k}{k_{m}}})+(18\sqrt{k}+10N)\sqrt{8k_{m}(\sigma^{2}+\sigma^{2}_{R})}\right)
    \end{aligned}
\end{equation*}
In order to get to a big $O$ notation, we write the various terms in the previous equation explicitly and ignore logarithmic terms. Note that $k_{m}$ is assume to be a logarithm term, and we can assume that $k = C\ceil{\epsilon_{b}^{2}/((1-\gamma)^2\epsilon_{s}^2)}$ for some constant $C$ (this does not violate assumption \ref{assump:main}). Therefore
\begin{equation*}
    \begin{aligned}
            \sum_{i=1}^{N}\sum_{t=0}^{\infty}\epsilon_{e}^{i}(t) = \widetilde{O}\left(\left(N\left(Q_{\max}+\sqrt{\sigma^{2}+\sigma^{2}_{R}}\right)+\frac{\sigma^{2}+\sigma^{2}_{R}}{\epsilon_{s}(1-\gamma)}\right)SA/(1-\gamma)\right)
    \end{aligned}
\end{equation*}
\end{proof}
Using Theorem \ref{theo:sample complexity}, the following theorem abandons the concentration bound assumption \ref{assump:Concentration Bound} over the learner-agent noise terms, and instead assumes the noise to be sub-Gaussian with a bounded mean and bounded parameter. We note that we can actually prove such a theorem for the more general case of noise with a bounded mean and variance. However, in such a case we replace a logarithmic term with a non-logarithmic one which worsens the bound. This is to be expected, as we have less information regarding the tails of a general distribution.  
\samplecomplexityvar*
\begin{proof}
From the definition of the noise in \eqref{eq: noise decomposition} and from assumption \ref{assump: Noise_assumptions2}, we have that $b^{i}(s,a)$ is subgaussian with
\begin{equation*}
    \begin{gathered}
              P(|b^{i}(\overline{s},\overline{a})|\geq n_{c}^{i}) \leq 2e^{-\frac{n^{2}_{c}}{2\sigma^{2}_{c}}}
    \end{gathered}
\end{equation*}
Setting
\begin{equation*}
    \begin{gathered}
              2e^{-\frac{n^{2}_{c}}{2\sigma^{2}_{c}}} = \frac{\delta}{12N\ceil{1+\log_{2}\frac{k}{k_{m}}}(SA)^{3}}
    \end{gathered}
\end{equation*}
results in $n_{c}^{i}=\sigma^{i}_{c}\sqrt{2\ln\left(\frac{24N\ceil{1+\log_{2}\frac{k}{k_{m}}}(SA)^{3}}{\delta}\right)}$. Now given the event that $|b^{i}(\overline{s},\overline{a})|\leq n_{c}^{i}$, we know that $|b^{i}(\overline{s},\overline{a})+m^{i}(\overline{s},\overline{a})|\leq n_{c}^{i}+\Delta Q^{i}_{c}$. Therefore
\begin{equation*}
    \begin{aligned}
              P\left\{|b^{i}(\overline{s},\overline{a})+m^{i}(\overline{s},\overline{a})|\geq n_{c}^{i}+\Delta Q^{i}_{c}\right\} \leq
              P(|b^{i}(\overline{s},\overline{a})|\geq n_{c}^{i}) \leq
              \frac{\delta}{12N\ceil{1+\log_{2}\frac{k}{k_{m}}}(SA)^{3}}
    \end{aligned}
\end{equation*}
so that the total equivalent noise $n=b+m$ is of mean $0$ and has the right probability bound. Substituting this back in the bound for Theorem \ref{theo:sample complexity} concludes our proof. 
\end{proof}
Compared to Theorem \ref{theo:sample complexity}, Theorem \ref{theo:sample complexity2} requires information about nothing but the first and second moments of the noise, enables a bounded noise term $m^{i}(s,a)$ to be present, and suggests that the sample complexity bound is proportional to a bound on the the sub-Gaussian parameter of the noise, and on the bound on $m^{i}(s,a)$.
\section{Estimation methods proofs and properties}
\label{sec:Appendix B: Estimation methods proofs and properties}
In this section, we provide proofs and further information about the optimal linear estimation theorems from Section \ref{sec:Estimation methods}. The proofs all consist of finding the parameter of the equivalent sub-Gaussian noise resulting from the weighted sum \eqref{eq:ApproxSet}, and minimizing it with regard to the estimation weights. For simplicity and without loss of generalization, we will present the proofs for agent $N$ as the receiving agent, and denote $\sigma^{N}_{L} \triangleq \sigma_{L},\sigma^{j,N}_{A} \triangleq \sigma_{A,j}, \sigma_{c}^{N} \triangleq \sigma_{c}$. 
\subsection{General additive noise model}
\label{General additive noise model}
\nonidenticaladditivenoise*
\begin{proof}
We know that a linear sum of independent sub-Gaussian random variables is a sub-Gaussian with the sum of their collective parameters. Furthermore, since the weights all sum up to $1$, by \eqref{eq: noise decomposition},\eqref{eq: learner and agent noise} \eqref{eq:ApproxSet} we can write
\begin{equation*}
\begin{aligned}
        \hat{Q}^{N}(s,a) &= Q(s,a)+n^{N}(s,a)\\
        n^{N}(s,a) &= w_{NN}n^{N}_{L}(s,a) + \sum_{j=1}^{d_{N}}w_{jN}n^{j,N}_{A}(s,a)
\end{aligned}
\end{equation*}
with $En^{N}(s,a)=0$ and parameter
\begin{equation*}
\begin{aligned}
        \sigma_{c}^{2}(\mathbf{w}) = \sum_{j=1}^{d_{N}}w_{jN}^{2}\left(\sigma_{L}^{2}+\sigma_{A,j}^{2}\right)+\left(1-\sum_{j=1}^{d_{N}}w_{jN}\right)^{2}\sigma_{L}^{2}~.
\end{aligned}
\end{equation*}
This means the sample complexity bound of Theorem \ref{theo:sample complexity} holds. Since the term $\sigma_{c}(\mathbf{w})$ is the only one depending on the weights in the above-mentioned bound, minimizing this term over the weights results in optimizing the sample complexity bound. 

Calculating the derivative, we get that 
\begin{equation*}
\frac{\partial \sigma_{c}^{2}(\mathbf{w})}{\partial \mathbf{w}} = A\mathbf{w}-\boldsymbol{\sigma}_{L}~,
\end{equation*}
meaning that the original function can be written in the following form
\begin{equation*}
\sigma_{c}^{2}(\mathbf{w}) = \frac{1}{2}\left(\mathbf{w}-A^{-1}\boldsymbol{\sigma}_{L}\right)^{T}A\left(\mathbf{w}-A^{-1}\boldsymbol{\sigma}_{L}\right)+C
\end{equation*}
where C is some constant. We can use the Matrix determinant lemma from \citep{detlemma} (where $X$ is a matrix and $\mathbf{u},\mathbf{v}$ are some vectors) $\mathrm{det}\left(X+\mathbf{u}\mathbf{v}^{T}\right)= \left(1+\mathbf{v}^{T}X^{-1}\mathbf{u}\right)\mathrm{det}(X)$ to show that the principle minors of A (which are all of the same form) are all positive, meaning that A is positive definite by Sylvester's criterion. This means that $\sigma_{c}(\mathbf{w})$ has a single global minimum at
\begin{equation*}
\label{optimal w]}
\mathbf{w}^{*}=A^{-1}\boldsymbol{\sigma}_{L}~. 
\end{equation*}
Now, by substituting $\sigma_{c}^{2}(\mathbf{0})=\sigma_{L}^{2}=\frac{1}{2}\left(A^{-1}\boldsymbol{\sigma}_{L}\right)^{T}A\left(A^{-1}\boldsymbol{\sigma}_{L}\right)+C$ we can easily see that  
\begin{equation*}
\sigma_{c}^{*} \triangleq \sigma_{c}(\mathbf{w}^{*}) = \sigma_{L}\sqrt{w_{NN}^{*}}~.
\end{equation*}
\end{proof}
\paragraph{Properties of the optimal solution}
\label{properties of the optimal solution}
We list some properties exhibited by the general optimal solution.
\begin{itemize}
	\item We have that $\sigma_{c}(\mathbf{w}) \geq 0 $ by definition, and also that $\sigma_{c}(\mathbf{w}^{*}) \leq \sigma_{c}(\mathbf{0})=\sigma_{L}^{2}$, and therefore $0 \leq w_{NN}^{*} \leq 1$. This also means that the optimal solution ensures that we get a lower sample complexity compared to the case without communication, where the variance is $\sigma_{L}^{2}$. 
	\item The optimal solution results in $\sigma_{c}^{*}$ smaller than the case of a uniform average.
	\item We can furthermore show that 
	\begin{equation*}
	\forall j \in\{1,...,d_{N}\}: w_{jN}^{*} \in [0,1] 
	\end{equation*}
	which is intuitive. 
	\item From the mathematical form of $\sigma_{c}(\mathbf{w})$, it is easy to deduce that the larger the variance $\sigma_{A,j}^{2}$ is, the smaller the corresponding optimal weight $w_{jN}^{*}$ gets. 
	\item In the case where $\sigma_{A,j}=\sigma_{A,l}$ for some $j$ and $l$, we have from the symmetry property of $\sigma_{c}(\mathbf{w})$ that $w_{jN}^{*}=w_{lN}^{*}$.
\end{itemize}
Let us now also consider two extreme cases and see whether our solution corresponds to intuition.
\begin{itemize}
	\item In the case where $\sigma_{A,j} \gg \sigma_{L}$ for some $j \in\{1,...,d_{N}\}$, it is evident from both intuition and the explicit solution that $w_{jN}^{*} \ll 1$. 
	\item In the case where $\sigma_{L} \gg \sigma_{A,j}$ for all $j \in\{1,...,d_{N}\}$, we have that \begin{equation*}
	\sigma_{c}(\mathbf{w}) = \sum_{j=1}^{d_{N}}w_{jN}^{2}\sigma_{L}^{2}+\left(1-\sum_{j=1}^{d_{N}}w_{jN}\right)^{2}\sigma_{L}^{2}
	\end{equation*} 
	which is a case where all the weights are identical. By intuition and by substituting in the explicit solution, we have that $w_{jN}^{*}=1/(d_{N}+1)$ for all $j \in\{1,...,d_{N}\}$.
\end{itemize}
\paragraph{Special examples}
\label{special examples}
We exemplify the properties introduced previously over simple cases.
\begin{enumerate}
	\item First, let us consider the case of three fully-connected agents. By calculating the optimal solution we can learn about the effects difference noise variances have on the solution. We can see in Figure \ref{fig:two_agents_optimal_solution} that, as expected, the self optimal weight $w^{*}_{33}$ (proportional to $\sigma_{c}^{*}$ itself) decreases as the inter-agent noise parameters $\frac{\sigma_{A,1}^{2}}{\sigma_{L}^{2}},\frac{\sigma_{A,2}^{2}}{\sigma_{L}^{2}}$ decrease, meaning that it is beneficial to not only use the less noisy $Q$-function received from the learner. The optimal weight for the case of no inter-agent noise is $w^{*}_{33}=\frac{1}{3}$, but the larger the noise between the agents is, the better it is to rely more on the less-noisy $Q$-function. We can also see that the larger the relation $\frac{\sigma_{A,1}^{2}}{\sigma_{A,2}^{2}}$ gets, the smaller $\frac{w_{13}^{*}}{w_{23}^{*}}$ becomes, as explained earlier.
\item In order to describe the dependence of the optimal weights on the number of agents, we examine a case in which there are 2 groups of agents of size $N_{A}$, the first one contains agents whose noise parameters are all $\sigma_{A,1}^{2}$, and the second one with noise parameters all equal to $\sigma_{A,2}^{2}$ (all the agents are fully connected). This is an extension of the previous case in which $N_{A}=1$. Besides these two groups, agent number $N$ on which we are focusing is also present, so there are overall $N=2N_{A}+1$ agents. 
In Figure \ref{fig:N_agents_optimal_solution}A, we can see that the more agents there are - the smaller $w_{33}^{*}$ becomes, meaning that it is beneficial to rely on more agents. Since we keep $\sigma_{A,2}=\sigma_{L}$ constant in this case, we see that as the ratio $\frac{\sigma_{A,1}^{2}}{\sigma_{A,2}^{2}}$ increases, $w_{33}^{*}$ increases as well as explained before. In Figure \ref{fig:N_agents_optimal_solution}B we can see that the ratio $\frac{w^{*}_{13}}{w^{*}_{23}}$ remains constant as long as  $\frac{\sigma_{A,1}^{2}}{\sigma_{A,2}^{2}}$ does not change, even for a varying number of agents. 
\end{enumerate}
\begin{figure}[h]
    \includegraphics[width=0.9\textwidth]{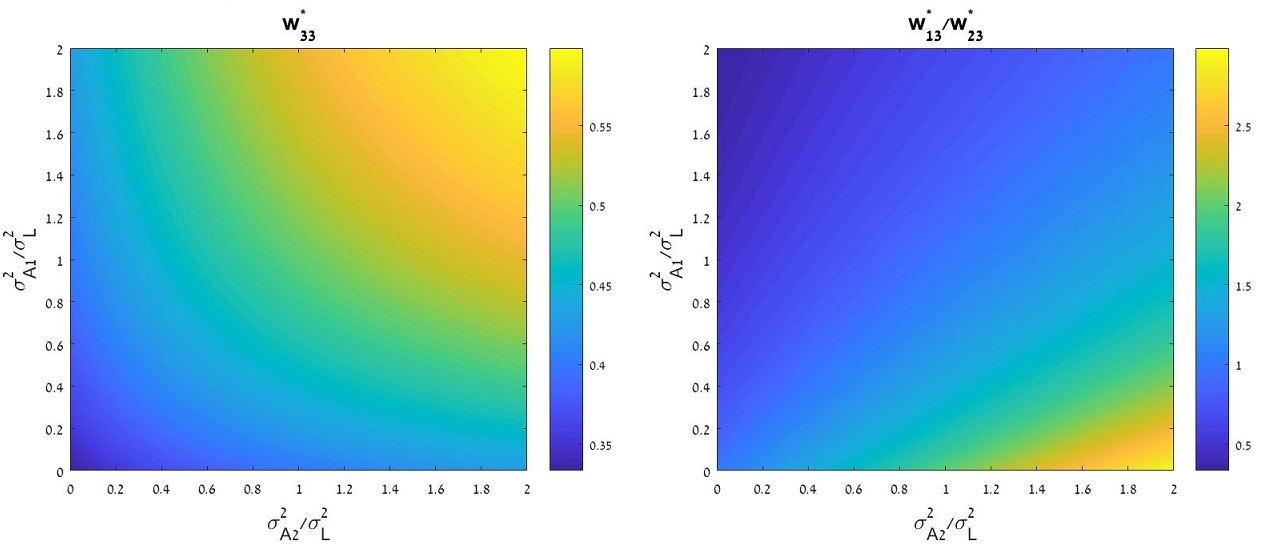}
      \centering
      \caption{\label{fig:two_agents_optimal_solution}Optimal weights for the fully connected 3-agent scenario, as a function of the environment noise parameters. We set $\sigma_{L}=1$ here. \textbf{A:} Optimal weight $w^{*}_{33}$, which is proportional to $\sigma_{c}^{*}$, as a function of $\frac{\sigma_{A,2}^{2}}{\sigma_{L}^{2}},\frac{\sigma_{A,1}^{2}}{\sigma_{L}^{2}}$ \textbf{B:} the ratio $\frac{w^{*}_{13}}{w^{*}_{23}}$ as a function of $\frac{\sigma_{A,2}^{2}}{\sigma_{L}^{2}},\frac{\sigma_{A,1}^{2}}{\sigma_{L}^{2}}$.}
\end{figure}
\begin{figure}[h]
    \includegraphics[width=0.9\textwidth]{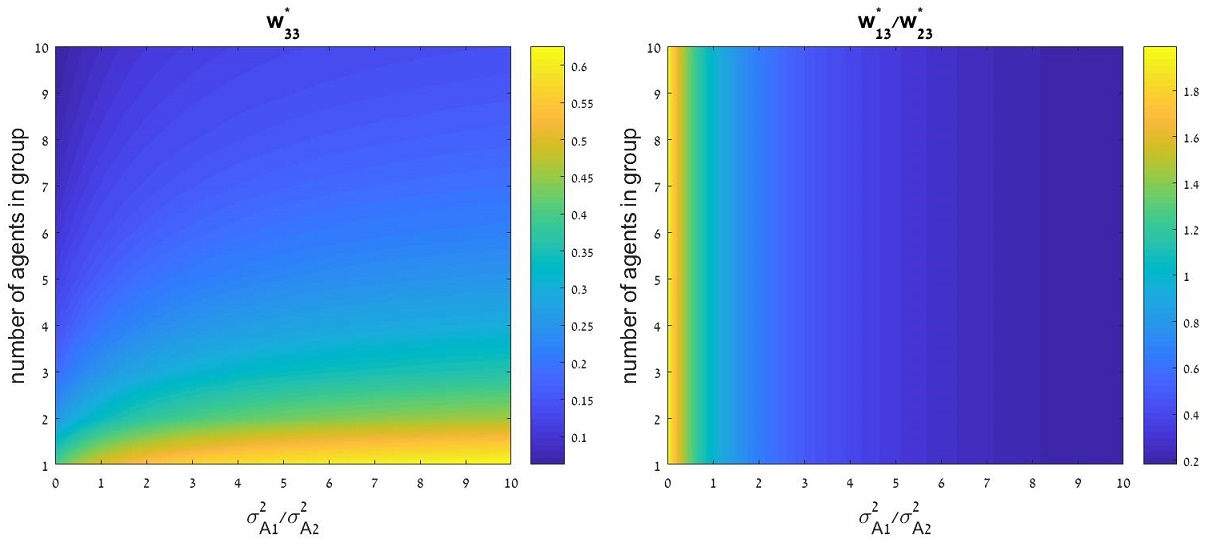}
      \centering
      \caption{\label{fig:N_agents_optimal_solution}Optimal weights for the fully-connected 2-groups scenario, as a function of the environment noise parameters and the number of agents in a group. We set $\sigma_{L}=\sigma_{A,2}=1$ here. \textbf{A:} Optimal weight $w^{*}_{33}$, which is proportional to $\sigma_{c}^{*}$, as a function of $\frac{\sigma_{A,1}^{2}}{\sigma_{A,2}^{2}}$ and $N_{A}$, the number of agents in a group. \textbf{B:} the ratio $\frac{w^{*}_{13}}{w^{*}_{23}}$ as a function of $\frac{\sigma_{A,1}^{2}}{\sigma_{A,2}^{2}}$ and $N_{A}$.}
\end{figure}
\subsection{Same parameters additive noise model}
\label{Same parameters additive noise model}
We present a special case of equal noise parameters. Theorem \ref{theo: identical additive noise} is a direct consequence of Theorem \ref{theo: nonidentical additive noise}. 
\identicaladditivenoise*
\begin{proof}
By substituting $\forall i,j: \sigma^{j}_{L}=\sigma_{L},\sigma^{j,i}_{A}=\sigma_{A}$ for the optimal weights of Theorem \ref{theo: nonidentical additive noise}, we get 
\begin{equation*}
    \begin{gathered}
      w_{ii}^{*} = \frac{\sigma^{2}_{L}+\sigma^{2}_{A}}{(d+1)\sigma^{2}_{L}+\sigma^{2}_{A}}=\frac{1}{1+\frac{d}{1+\sigma^{2}_{A}/\sigma^{2}_{L}}} \quad ; \quad 
       w^{*}_{ji} =\frac{1-w_{ii}^{*}}{d} \quad \forall j \in\{1,...,d_{i}\}
    \end{gathered}
\end{equation*}
where weights of $Q$-functions with the same noise parameter all have the same weight from symmetry. 

For the second part of the proof, by substituting $\forall j \in\{1,...,d_{i}\}:w_{ji}=\frac{1}{d+1}$ and $w_{ii}=\frac{1}{d+1}$ in $\sigma_{c}(\mathbf{w})$ from the proof of Theorem \ref{theo: nonidentical additive noise}, we get
\begin{equation*}
\begin{gathered}
        \widetilde{\sigma_{c}}^{2} = \frac{\sigma_{L}^{2}}{d+1}+\frac{d}{(d+1)^{2}}\sigma_{A}^{2}
\end{gathered}
\end{equation*}
And we can see that using a uniform average is preferable to using the less noisy $Q$-function from the learner alone, if and only if $\widetilde{\sigma_{c}}^{2} \leq \sigma^{2}_{L}$ which leads to $d+1 \geq \frac{\sigma^{2}_{A}}{\sigma^{2}_{L}}$. 
\end{proof}

\subsection{Quantization noise model}
\label{Quantization noise model}
We now prove Theorem \ref{theo: quantization noise}. 
\quannoise*
\begin{proof}
Without the loss of generality, we will again prove the theorem for agent $N$.
\paragraph{First part:}
We  partition the noise in the weighted $Q$-function $\hat{Q}^{N}$ to two terms:
\begin{equation*}
\begin{aligned}
        \hat{Q}^{N}(s,a) &= Q(s,a)+b^{N}(s,a)+m^{N}(s,a)\\
        b^{N}(s,a) &\triangleq w_{NN}b_{L}^{N}(s,a) + \sum_{j = 1}^{d_{N}}w_{jN}(b_{L}^{j}+b^{j,N}_{A}(s,a))\\
        m^{N}(s,a) &\triangleq w_{NN}m_{L}^{N}(s,a) + \sum_{j = 1}^{d_{N}}w_{jN}(m_{L}^{j}+m^{j,N}_{A}(s,a))
\end{aligned}
\end{equation*}
We know that a linear sum of independent sub-Gaussian random variables is a sub-Gaussian with the sum of their collective parameters. Therefore, $b^{N}(s,a)$ is a mean $0$ sub-Gaussian with some parameter $\sigma_{c}^{2}$ as in Theorem \ref{theo: nonidentical additive noise}. Regarding the quantization noise, we have 
\begin{equation*}
\begin{gathered}
        |m^{N}(s,a)| \leq \Delta Q^{N}|w_{NN}|+ \sum_{j = 1}^{d_{N}}\left(\Delta Q^{j}+\Delta Q^{j,N}\right)|w_{jN}| \triangleq \Delta Q_{c}
\end{gathered}
\end{equation*}
Therefore, the conditions of Theorem \ref{theo:sample complexity} hold.
\paragraph{Second Part:}
In the case of equal noise properties, due to the symmetry of the problem, we know that the optimal weights will be such that the $Q$-function that agent $N$ receives from the learner has some weight $w \triangleq w_{NN}$, and the rest have weights of identical values $\frac{1-w}{d}$. Therefore, $b^{N}(s,a)$ is a mean $0$ sub-Gaussian with parameter $\sigma_{c}^{2} \leq w^{2}\sigma_{L}^{2}+\frac{(1-w)^{2}}{d}(\sigma_{L}^{2}+\sigma_{A}^{2})=\sigma_{L}^{2}\left(w^{2}+2\frac{(1-w)^{2}}{d}\right)$. And regarding the quantization noise:
\begin{equation*}
\begin{gathered}
        |m^{N}(s,a)| \leq \Delta Q \left(|w|+\sum_{j = 1}^{d}|\frac{1-w}{d}|(1+1)\right)=\Delta Q \left(|w|+2|1-w|\right)
\end{gathered}
\end{equation*}
In order to find the optimal weight $w^{*}$ that minimizes the sample complexity bound, define
\begin{equation*}
    \begin{gathered}
            g(w) \triangleq \sigma_{L}f\sqrt{w^{2}+2\frac{(1-w)^{2}}{d}} +\Delta Q \left(|w|+2|1-w|\right)
    \end{gathered}
\end{equation*}
Where the function $f$ is defined in Theorem \ref{theo:sample complexity}. From the sample complexity bound
\begin{equation*}
    \begin{gathered}
       V^{\widetilde{\pi}_{i}}(s_{t,i}) \geq V^{*}(s_{t,i}) - \frac{2\epsilon_{a}+2(1+3\gamma)g(w)}{1-\gamma}-3\epsilon_{s} -\epsilon_{e}^{i}(t)
    \end{gathered}
\end{equation*}
It is evident that the function $g(w)$ is the only term depending on the weight $w$, such that minimizing $g(w)$ over $w$ will lead to maximization of the bound. 

Viewing $g(w)$, we can learn a few properties regarding the optimal solution. Notice that the first term is a parabola-like function with a minimum value at $w=\frac{2}{d+2}$ as we have seen in Theorem \ref{theo: identical additive noise}, and the second term has a few cases of discontinuous change. 
\begin{itemize}
	\item For $w<\frac{2}{d+2}$ we have a sum of two decreasing function, therefore $g(w)$ itself decreases. 
	\item For $w>1$, we have a sum of two increasing functions - and therefore $g(w)$ is increasing. 
	\item Therefore, there is a single minimal value of $g(w)$ in the interval $w \in \left[\frac{2}{d+2},1\right]$.   
\end{itemize}
By calculating the derivative of $f$ in $[0,1]$ and finding the value of $w$ for which it vanishes, we have that there is a single solution in this interval, satisfying
\begin{equation}
    \begin{gathered}
    \label{eq:optimal w}
            w=\frac{2}{d+2}\left(1+\frac{d}{\sqrt{(d+2)(2\sigma_{L}^{2}f^{2}/\Delta Q)^{2}-2d}}\right)
    \end{gathered}
\end{equation}
Under the condition that $\frac{f\sigma_{L}}{\Delta Q} \geq \frac{d}{\sqrt{d+2}}$. Otherwise, there is no minimal value in this interval. Furthermore, we can see that the expression for the solution is decreasing with $f$, and that substituting $f\sigma_{L}= \Delta Q$ results in the minimal value being $w=1$.
Therefore, we conclude that:
\begin{itemize}
	\item For $f\sigma_{L}\leq \Delta Q$ the optimal solution is $w^{*}=1$ (either the minimal value is outside of $[0,1]$ and thus contradicts our initial assumption, or there is no minimal value at all and $g(w)$ is decreasing in $[0,1]$).
	\item For $f\sigma_{L}\geq \Delta Q$, the minimal value of $g(w)$ is achieved at $w$ from \eqref{eq:optimal w}, which is indeed inside the interval $\left[\frac{2}{2+d},1\right]$, and gets closer to $0$ as $d$ grows larger.
\end{itemize}
\end{proof}
\section{Computational results}
\label{sec:Appendix C: Computational results}
In this section we give more details about the computational results shown in Figure \ref{fig:simulation}.
In this simulation, 4 agents are randomly initiated in an unknown warp-around $5\times 5$ grid world, and get a reward of $1$ whenever they land on the top-right corner. All other grid locations contain a reward of $0$. When an agent reaches the top-right corner, it is re-initiated at a random location in the grid world. $Q$-functions are noisy with an additive Gaussian noise, under the conditions of Theorem \ref{theo: identical additive noise}, and we assume no reward noise. The agents communicate via a fully-connected communication network $\Gamma$. The results are shown for $150$ parallel experiments, $10$ episodes and 50 steps per episode, where at the beginning of an episode - each agent is re-initiated at a random location in the grid world. We calculate the average accumulated reward and standard deviation for an episode per agent, where the average is calculated over parallel experiments and over all agents (continuous lines). We also use a simulated agent that can move greedily in the grid world using the exact non-noisy $Q$-function (dashed lines). This is done to separate the effects of the noise from those of exploration. In Figure \ref{fig:simulation}, the learner-agent communication noise variance is fixed at $\sigma^{2}_{L}=0.1$, and each one of the figures A,B,C represents a different agent-agent noise variance $\sigma^{2}_{A}$, in order to show the tradeoff between the $\sigma^{2}_{A}/\sigma^{2}_{L}$ ratio and the estimation weighting scheme. The red lines represent a weight $w=1$ for the $Q$-function received from the learner, meaning there is no averaging. The blue lines represent a uniform averaging $w=0.25$, and the black lines represent the optimal weighting scheme suggested by Theorem \ref{theo: identical additive noise}. We note that for an algorithm to be able to perform these estimation method, prior knowledge over the noise parameters is needed. Given such knowledge, each agent can decide on its optimal weights before the exploration begins and use them throughout the whole loop. The green line represents the more realistic case where the algorithm has no information about the parameters of the environmental noise, and estimates the variance on the go in order to use it at the optimal weighting scheme of Theorem \ref{theo: identical additive noise}. We use an estimation scheme similar to that suggested by \citep{VarEst}, such that for agent $i$
\begin{equation*}
    \begin{aligned}
            \hat{w}_{ii}(t) &= \frac{(\hat{\sigma}^{i}_{LA}(t-1))^{2}}{(\hat{\sigma}^{i}_{LA}(t-1))^{2}+(\hat{\sigma}^{i}_{L}(t-1))^{2}} \\ 
            (\hat{\sigma}^{i}_{L}(t))^{2} &= (\hat{\sigma}^{i}_{L}(t-1))^{2}+\frac{\frac{1}{SA-1}\sum_{s,a}\left(\widetilde{Q}^{i}_{L}(s,a)-Q(s,a)\right)^{2}-(\hat{\sigma}^{i}_{L}(t-1))^{2}}{t}\\
            (\hat{\sigma}^{i}_{LA}(t))^{2} &= (\hat{\sigma}^{i}_{LA}(t-1))^{2}+\frac{\frac{1}{SA(N-1)-1}\sum_{s,a,N}\left(\widetilde{Q}^{j,i}_{A}(s,a)-Q(s,a)\right)^{2}-(\hat{\sigma}^{i}_{LA}(t-1))^{2}}{t}
    \end{aligned}
\end{equation*}
$(\hat{\sigma}^{i}_{L}(t))^{2}$ represents the estimation of the variance of the learner-agent noise  ($\sigma^{2}_{L}$) at step $t$, and $(\hat{\sigma}^{i}_{LA}(t))^{2}$ represents the estimation of the agent-agent communication noise variance ($\sigma^{2}_{A}+\sigma^{2}_{L}$) at step $t$, such that the formula for $\hat{w}_{ii}(t)$ is an adaptive version of the optimal weight. The estimation of $(\hat{\sigma}^{i}_{L}(t))^{2}$ in each step is done by using a weighted average of the current estimate with the past one. The estimate given the current $Q$-function is an empiric average over all values in $\widetilde{Q}^{i}_{L}$. $(\hat{\sigma}^{i}_{LA}(t))^{2}$ is calculated similarly, but uses all of the $N-1$ values in $\{\widetilde{Q}^{j,i}_{A}\}$. 
\begin{figure}[h]
    \includegraphics[width=0.5\textwidth]{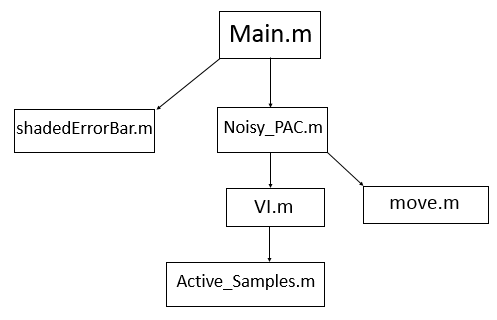}
      \centering
      \caption{\label{fig:code flow}This figure illustrates the code dependencies, for the code producing Figure \ref{fig:simulation}}.
\end{figure}
We use a version of Algorithm \ref{alg:Communication noise tolerant PAC Exploration} that accumulates samples for each state-action instead of replacing them with new ones whenever the number of active samples increases. We use $\epsilon_{a}=10^{-7}, \epsilon_{b}=0.1, k=9, k_{m}=3, \gamma=0.98,Q_{\max}=\frac{R_{max}}{1-\gamma}$ and $30$ Bellman iterations at most each time we perform them. 

MATLAB Code is available at \url{https://github.com/ravehor92/ALT2020-Raveh}, and contains the following sub-codes, while the dependencies are shown in Figure \ref{fig:code flow}.
\begin{itemize}
    \item \emph{main.m:} The main code generating Figure \ref{fig:simulation}.
    \item \emph{Noisy\_PAC.m:} performs a parallel experiments of agents exploring in a grid world, for the same parameters. 
    \item \emph{VI.m:} performs Value Iterations over a given Q function. 
    \item \emph{move:.m} given a state and actions, progress to the next state and receive a reward.
    \item \emph{Active\_Samples.m:} Number of active samples for a given state-action (samples that are currently being used for value iteration). 
    \item \emph{shadedErrorBar.m:} Creates a shaded error bar.
\end{itemize}
\end{sloppy}
\end{document}